\documentclass[11pt]{article}

\usepackage{amsfonts,amsthm}
\usepackage{amssymb,amscd}
\usepackage[cmex10]{amsmath}
\usepackage{mathtools} 
\usepackage{ifthen}

\usepackage{endnotes}

\usepackage{color}
\input{rh_ss_def.sty}
\usepackage{bm}
\newcommand{\vla}{ {\bm \lambda}} 

\newcommand{\vnu}{ {\bm \nu}} 

\usepackage{pgf,tikz}
\usepackage{pgfplots}
\usepackage{amsmath}
\usetikzlibrary{calc,shadows,backgrounds}
\usetikzlibrary{pgfplots.groupplots}
\usepackage{multirow}

\usepackage{placeins}

\pgfplotsset{width=7cm,compat=1.3}

\usepackage[hyphens]{url}

\usepackage[hyperindex,breaklinks]{hyperref}
\hypersetup{
    bookmarks=true,         
    unicode=false,          
    pdftoolbar=true,        
    pdfmenubar=true,        
    pdffitwindow=false,     
    pdfstartview={FitH},    
    pdfauthor={Reinhard Heckel},     
    pdfsubject={Subject},   
    pdfcreator={Reinhard Heckel},   
    pdfproducer={Producer}, 
    pdfnewwindow=true,      
    colorlinks=false,       
    linkcolor=red,          
    citecolor=green,        
    filecolor=magenta,      
    urlcolor=cyan           
}

\usepackage{setspace}

\usepackage{tabularx}

\newtheorem{theorem}{Theorem}
\newtheorem{lemma}{Lemma}

\newtheorem{corollary}{Corollary}
\newtheorem{proposition}{Proposition}

\renewcommand\sgn{\mathrm{sign}}

\newcommand\blfootnote[1]{
  \begingroup
  \renewcommand\thefootnote{}\footnote{#1}
  \addtocounter{footnote}{-1}
  \endgroup
}

\usepackage{xcolor}
\newcommand\crule[3][black]{\textcolor{#1}{\rule{#2}{#3}}}

\usepackage[applemac]{inputenc}

\begin{document}

\title{Dimensionality-reduced subspace clustering}

\author{Reinhard Heckel, Michael Tschannen,  and Helmut B\"olcskei
\blfootnote{ \hspace{-0.25cm}\crule[white]{0.25cm}{0.25cm}  R. Heckel was with the Department of Information Technology and Electrical Engineering, ETH Zurich, Switzerland, and is now with IBM Research, Zurich, Switzerland (e-mail: reinhard.heckel@gmail.com). M.~Tschannen and H.~B\"olcskei are with the Department of Information Technology and Electrical Engineering, ETH Zurich, Switzerland (e-mail: michaelt@nari.ee.ethz.ch; boelcskei@nari.ee.ethz.ch).}
\blfootnote{\hspace{-0.25cm}\crule[white]{0.25cm}{0.25cm} Part of this paper was presented at the 2014 IEEE International Symposium on Information Theory (ISIT) \cite{heckel_subspace_2014}.}
}

\maketitle

\begin{abstract}
Subspace clustering refers to the problem of clustering unlabeled high-dimensional data points into a union of low-dimensional linear subspaces, whose number, orientations, and dimensions are all unknown. In practice one may have access to dimensionality-reduced observations of the data only, resulting, e.g., from undersampling due to complexity and speed constraints on the acquisition device or mechanism. More pertinently, even if the high-dimensional data set is available it is often desirable to first project the data points into a lower-dimensional space and to perform  clustering there; this reduces storage requirements and computational cost. The purpose of this paper is to quantify the impact of dimensionality reduction through random projection on the performance of three subspace clustering algorithms, all of which are based on principles from sparse signal recovery. Specifically, we analyze the thresholding based subspace clustering (TSC) algorithm, the sparse subspace clustering (SSC) algorithm, and an orthogonal matching pursuit variant thereof (SSC-OMP). We find, for all three algorithms, that dimensionality reduction down to the order of the subspace dimensions is possible without incurring significant performance degradation.  
Moreover, these results are order-wise optimal in the sense that reducing the dimensionality further leads to a fundamentally ill-posed clustering problem. 
Our findings carry over to the noisy case as illustrated through analytical results for TSC and simulations for SSC and SSC-OMP. 
Extensive experiments on synthetic and real data complement our theoretical findings. 
\end{abstract}

\renewcommand{\cS}{\mathcal S}
\renewcommand{\S}{S}
\newcommand{\Ss}[1]{\mathcal{S}_{#1}}
\newcommand{\Rn}[1]{\mathbb{R}^{#1}}
\newcommand{\Us}[1]{\mathbf{U}^{(#1)}}
\newcommand{\Ph}{\mathbf{\Phi}}
\newcommand{\N}[1]{\mathcal{N}(0,#1)}
\newcommand{\mat}[1]{\mathbf{#1}}
\newcommand{\matt}[1]{\mathbf{\tilde{#1}}}

\newcommand{\Bhati}[1]{\boldsymbol{\hat{\beta}}_{#1}}
\newcommand{\bbeta}{\boldsymbol{\beta}}

\newcommand{\mPh}{ \mathbf{\Phi} }
\renewcommand{\l}{\ell}
\newcommand{\pit}{s} 

\newcommand{\Ea}{\mathcal E_5}
\newcommand{\Eb}{\mathcal E_2}
\newcommand{\Ec}{\mathcal E_1^{(\l,i)}}
\newcommand{\Ed}{\mathcal E_4^{(\l,i,\pit)}}
\newcommand{\Ee}{\mathcal E_3^{(\l,i,\pit)}}
\newcommand{\Ecc}{\comp{\mathcal E}_1^{(\l,i)}}
\newcommand{\Edc}{\comp{\mathcal E}_4^{(\l,i,\pit)}}
\newcommand{\Eec}{\comp{\mathcal E}_3^{(\l,i,\pit)}}
\newcommand{\Estar}{\mathcal E_\star^{(\l,i,\pit)}}
\newcommand{\Estartot}{\mathcal E_\star}
\newcommand{\sigmin}[1]{\sigma_{\min} \! \left( #1 \right)}
\newcommand{\sigmax}[1]{\sigma_{\max} \! \left( #1 \right)}

\newcommand{\abs}[1]{\left\lvert #1 \right\rvert}

\section{Introduction}
One of the major challenges in modern data analysis is to find low-dimensional structure in large high-dimensional data sets. A prevalent low-dimensional structure is that of data points lying in a union of (low-dimensional) subspaces. 
The problem of extracting such a structure from a given data set can be formalized as follows. 
Consider the (high-dimensional) set $\YS$ of points in $\reals^m$ and assume that 
$
\YS = \YS_1 \cup \, \ldots \, \cup  \YS_L
$,  
where the points in $\YS_\l$ lie in a linear 
subspace $\cS_\l$ of $\reals^m$. 
The association of the data points to the sets $\YS_
\l$, the orientations, dimensions, and the number of the subspaces $\cS_\l$ are all unknown. 
The problem of identifying the assignments of the points in $\YS$ to the $\YS_\l$ is referred to as subspace clustering \cite{vidal_subspace_2011} or hybrid linear modeling and has applications, inter alia, in unsupervised learning, image representation and segmentation, computer vision, and disease detection. 

In practice one may have access to dimensionality-reduced observations of $\YS$ only, resulting, e.g., from ``undersampling'' due to complexity and speed constraints on the acquisition device or mechanism. 
More pertinently, even if the data points in $\YS$ are directly accessible, it is often desirable to work on a dimensionality-reduced version of $\YS$ as this reduces data storage cost and leads to computational complexity savings.   
 The idea of reducing computational complexity through dimensionality reduction 
  appears, e.g., in \cite{vempala_random_2005} in a general context, and for subspace clustering in the experiments reported in \cite{zhang_hybrid_2012,elhamifar_sparse_2013}.  Dimensionality reduction also has a privacy-enhancing effect in the sense that no access to the original data is needed for processing \cite{liu_random_2006}. 

Dimensionality reduction will, in general, come at the cost of clustering performance. The purpose of this paper is to analytically characterize this performance degradation for three subspace clustering algorithms, namely  thresholding-based subspace clustering (TSC) \cite{heckel_robust_2013}, sparse subspace clustering (SSC) \cite{elhamifar_sparse_2009,elhamifar_sparse_2013}, and SSC-orthogonal matching pursuit (SSC-OMP) \cite{dyer_greedy_2013}.
The common theme underlying these three algorithms is that they apply spectral clustering to an adjacency matrix constructed from sparse representations of the data points, obtained through a nearest neighbor search in the case of TSC, through $\l_1$-minimization for SSC, and through OMP in the case of SSC-OMP. While there are numerous further approaches to  subspace clustering (see \cite{vidal_subspace_2011} for an overview),  
we chose to study TSC, SSC, and SSC-OMP, as they belong to the small group of subspace clustering algorithms 
that are computationally tractable and succeed provably under nonrestrictive conditions  \cite{soltanolkotabi_geometric_2011,soltanolkotabi2014robust,elhamifar_sparse_2013,dyer_greedy_2013, you_sparse_2015, heckel_robust_2013}. 
Specifically, the results in   
\cite{heckel_robust_2013} for TSC, and in \cite{soltanolkotabi_geometric_2011,soltanolkotabi2014robust} for SSC 
show that TSC and SSC can succeed even when the  subspaces $\cS_\l$ intersect. The corresponding proof techniques, together with analytical performance guarantees for SSC-OMP developed in this paper, form the basis for our analytical characterization of the impact of dimensionality reduction on subspace clustering performance.

\paragraph{Formal problem statement and contributions.} Consider a set of $N$ data points $\YS \in \reals^m$, and assume that 
$
\YS = \YS_1 \cup \, \ldots \, \cup  \YS_L
$, 
where the points $\vy_i^{(\l)} \in \YS_\l, i\in \{1,\ldots,n_\l\}$, lie in a $d_\l$-dimensional linear subspace of $\reals^m$, denoted by $\cS_\l$. 
Neither the assignments of the points in $\YS$ to the sets $\YS_\l$ nor the subspaces $\cS_\l$ or the number of subspaces $L$ are known. 
Traditional subspace clustering operates on the data $\YS$ with the goal of segmenting it into the sets $\YS_\ell$. 
 Here, we assume, however, that clustering is performed on a dimensionality-reduced version of the points in $\YS$. Specifically, we employ the random projection method \cite{vempala_random_2005} by first applying the (same) realization of a random projection matrix $\mPh \in \reals^{p\times m}$ 
(typically $p  \ll m$) to each point in $\YS$ to obtain the set of dimensionality-reduced data points $\X$. Then, we declare the segmentation obtained by operating on 
$\X$ to be the segmentation of the data points in $\YS$. The realization of $\mPh$ does not need to be known.
There are two error sources that determine the performance of this approach,   
first, the error that would be obtained even if clustering was performed on the high-dimensional data set $\YS$ directly, second, and more pertinently, the error incurred by operating on dimensionality-reduced data. 
The former is quantified for TSC in \cite{heckel_robust_2013}, for SSC in \cite{soltanolkotabi_geometric_2011,soltanolkotabi2014robust}, and for SSC-OMP this paper develops corresponding new results.  
 Analytically characterizing the error incurred by  dimensionality reduction is the main contribution of this paper. 

While it is conceivable that TSC, which is based on thresholding inner products, exhibits graceful performance degradation as the data set's dimensionality is reduced through random projection, this is far from obvious for the $\ell_1$-minimization based SSC algorithm and the iterative SSC-OMP algorithm. 
We prove our main results by first deriving conditions for TSC, SSC, and SSC-OMP to ensure correct clustering of dimensionality-reduced data. While these conditions are general, they only become amenable to insightful interpretations once particularized  for a random data model, also used in  \cite{soltanolkotabi_geometric_2011,heckel_robust_2013}, that takes the subspace structure of the data set into account. 
The resulting clustering conditions make the impact of dimensionality reduction explicit and reveal a tradeoff between the affinity of the subspaces $\cS_\l$ and the amount of dimensionality reduction possible. 
Specifically, we find that all three algorithms succeed provably under quite generous conditions on the relative orientations of the subspaces $\cS_\l$, provided that the dimensionality is reduced no more than down to the largest subspace dimension $d_{\max} = \max_\ell d_\ell$. 
As the computational complexity associated with the construction of the adjacency matrix is essentially linear in the dimension of the ambient space, $m$, for all three algorithms, random projection reduces the complexity of this step by a factor of $m/d_{\max}$. 
These complexity savings translate into, possibly significant, run-time savings for the overall clustering algorithms (which include the spectral clustering step), in particular when $m$ is sufficiently large relative to $N$. 

We study the impact of noise---added to the high-dimensional data points---on clustering performance. 
For TSC, we derive a clustering condition which quantifies the tradeoff between the affinity of the subspaces $\cS_\l$ and the amount of dimensionality reduction possible, as a function of noise variance. 
Specifically, this condition allows us to conclude that TSC succeeds provably provided that---as in the noiseless case---the dimensionality is reduced to no more than down to the largest subspace dimension $d_{\max}$, and the noise variance is sufficiently small. 
%
An approach akin to that used for TSC can be applied to establish a similar clustering condition for SSC-OMP. 
The corresponding technical details are, however, significantly more involved and cumbersome. 
We therefore decided not to state the formal result.  
Regarding SSC, we remark that Wang et al.~\cite{wang_deterministic_2015} 
reported deterministic clustering conditions for the Lasso-version of SSC \cite{soltanolkotabi2014robust} applied to dimensionality-reduced noisy data. 
However, the corresponding results \cite[Lem.~16, Thm.~18]{wang_deterministic_2015} make the critical 
assumption of the signal part of the \emph{projected} noisy data being normalized, whereas the noise component remains un-normalized. 
It is difficult to see how one would realize this in practice, unless the noise realization is known perfectly, in which case the noise component could be removed which would take us back to the noiseless case. 
The results in \cite{wang_deterministic_2015} for noisy data therefore appear to be of limited practicality. 
While the statements in \cite{wang_deterministic_2015} may be particularized to the noiseless case, we note that corresponding results appeared in the conference version \cite{heckel_subspace_2014} of this paper before the publication of \cite{wang_deterministic_2015}. 

We note that our results, both for the noiseless and the noisy case, 
apply even when the subspaces $\cS_\l$ span the ambient space $\reals^m$. 
This follows from our clustering conditions depending on the \emph{pairwise} affinities between subspaces only, and pairwise affinities changing only moderately if the dimensionality is reduced down to no more than the order of the individual subspace dimensions. 
%

Another popular dimensionality reduction method is principal component analysis (PCA). 
However, when used in the context of subspace clustering, PCA allows dimensionality reduction down to the dimension of the overall span of the subspaces only, in general; this results in no dimensionality reduction at all 
when the subspaces $\cS_\l$ span the ambient space. 
To see this, consider the $L$ subspaces of dimension $1$ that correspond to the standard basis in $\reals^m$, i.e., the $\ell$-th subspace is spanned by the vector $\ve_\ell$ given by $[\ve_\l]_\l = 1$ and $[\ve_\l]_i = 0$,  for $i\neq \l$. 
Assuming that each of the data points in the data set under consideration, denoted by $\mY \in \reals^{m\times N}$, lies in one of these $L$ subspaces, the corresponding sample covariance matrix $\mY \transp{\mY}$ has non-zero entries only in its first $L$ main diagonal entries. 
The first $L$ principal components are therefore given by the vectors $\ve_\ell$. Reducing the dimensionality of the data set to below $L$ will result in certain data points being mapped to zero (owing to the orthogonality of the $\ve_\ell$). 
Moreover, PCA has computational complexity $O(Nm^2 + m^3)$ while random projection through Gaussian matrices and fast random projection matrices  \cite{ailon_almost_2013} has complexity $O(pmN)$ and $O(\log(m) m N)$, respectively, and is therefore computationally much less demanding. 
This is an important aspect as computational complexity is a major motivation for dimensionality reduction.

\paragraph{Notation.} We use lowercase boldface letters to denote (column) vectors and uppercase boldface letters to designate matrices. 
The superscript $\herm{}$ stands for transposition. 
For the vector $\vx$, $x_q$ denotes its $q$th entry and $\vx_\S$ is the subvector of $\vx$ with entries corresponding to the indices in the set $\S$. 
For the matrix $\mA$, $\mA_{ij}$ designates the entry in its $i$th row and $j$th column, $\mA_\S$ the matrix containing the columns of $\mA$ with indices in the set $\S$, 
 $\norm[2\to 2]{\mA} \defeq\;$ $\max_{\norm[2]{\vv} = 1  } \norm[2]{\mA \vv}$ its spectral norm, $\sigmin{\mA}$ its minimum singular value, and $\norm[F]{\mA} \defeq (\sum_{i,j} |\mA_{ij}|^2 )^{1/2}$ its Frobenius norm. 
 If $\mA$ has full column rank $\pinv{\mA} \defeq \inv{(\transp{\mA} \mA)} \transp{\mA}$ stands for its (left) pseudoinverse, and for $\mA$ with full row rank, $\pinv{\mA} \defeq\transp{\mA} \inv{(\mA \transp{\mA} )}$ is the (right) pseudoinverse. 
The identity matrix is denoted by $\mI$. $\log(\cdot)$ refers to the natural logarithm, $\acos(\cdot)$ is the inverse function of $\cos(\cdot)$, and $x \land y$ denotes the minimum of $x$ and $y$. 
The set $\{1, \ldots ,N\}$ is written as $[N]$. The cardinality of the set $\S$ is designated by $|\S|$ and its complement is $\comp{\S}$.
 $\mathcal N( \boldsymbol{\mu},\boldsymbol{\Sigma})$ stands for the distribution of a real Gaussian random vector with mean $\boldsymbol{\mu}$ and covariance matrix $\boldsymbol{\Sigma}$.
We write $X \sim Y$ to indicate that the random variables $X$ and $Y$ are equally distributed. 
For notational convenience, we use the following shorthands: $\max_\l$ for $\max_{\l \in [L]}$, $\max_{k\neq \l}$ for $\max_{k \in [L] \colon k\neq \l}$, and $\max_{k,\l \colon k\neq \l}$ for $\max_{k,\l  \in [L ]\colon  k\neq \l}$. 
The unit sphere in $\reals^m$ is $\US{m} \defeq \{ \vx \in \reals^m \colon \norm[2]{\vx} = 1 \}$. A subgraph $H$ of a graph $G$ is said to be connected if every pair of nodes in $H$ can be joined by a path along edges with nodes exclusively in $H$. A subgraph $H$ of $G$ is called a connected component of $G$ if $H$ is connected and if there are no edges between nodes in $H$ and the remaining nodes in $G$.

\vspace{-0.1cm}
\section{A brief review of TSC, SSC, and SSC-OMP}
\label{sec:AlgIntro}

We next briefly summarize the TSC \cite{heckel_robust_2013}, SSC \cite{elhamifar_sparse_2009,elhamifar_sparse_2013}, and SSC-OMP \cite{dyer_greedy_2013} algorithms. All three algorithms apply normalized spectral clustering \cite{luxburg_tutorial_2007} to an adjacency matrix $\mA$ built by finding a sparse representation of each data point in terms of the other data points. 
Specifically, TSC is based on least-squares representations in terms of nearest neighbors while 
SSC and SSC-OMP construct $\mA$ by finding sparse representations 
via $\ell_1$-minimization 
and OMP, respectively. 
Note that the focus in \cite{heckel_robust_2013} is on a version of TSC that uses a spherical distance measure between data points instead of least-squares regression coefficients to determine the entries of $\mA$. The analytical results presented here  
apply to both versions of TSC. 
We decided, however, to work with the least-squares version as this formulation better elucidates the sparsity aspect and thereby the relationship to SSC and SSC-OMP.

In order to emphasize that we consider all three  algorithms applied to dimensionality-reduced data, their descriptions will be in terms of the dimensionality-reduced data set $\X \subset \reals^p$. We furthermore assume that an estimate $\hat L$ of the number of subspaces $L$ is available. The estimation of $L$ from $\X$ is discussed later. 
We also note that the formulations of the TSC and SSC-OMP algorithms below assume that the data points in $\X$ are of comparable $\ell_2$-norm. This assumption is relevant for Step 1 in both cases and is not restrictive as the data points can be normalized prior to clustering.

\vspace{0.2cm} 
{\bf The TSC algorithm:} 
Given a set 
of $N$ data points 
 $\X$ in $\reals^p$, an estimate of the number of subspaces $\hat L$, 
 and the parameter $q$, perform the following steps: 
 
{\bf Step 1:} For every $\vx_j \in \X$, find the set $\S_j \subset [N] \! \setminus \! \{j\}$ of cardinality $q$ defined by
\begin{equation*}
\left| \innerprod{\vx_j}{ \vx_i} \right| \geq \left| \innerprod{\vx_j}{ \vx_k} \right|, \text{ for all }  i \in \S_j \text{ and all } k \notin \S_j,
\end{equation*}

and let $\vz_j$ be the coefficient vector corresponding to the minimum least-squares representation of $\vx_j$ in terms of $\vx_i, i \in \S_j$.
Specifically, set $(\vz_j)_{\S_j} = \arg \min_{\vz} \norm[2]{\vx_j - \mX_{\S_j} \vz  }$ (if multiple solutions exist, choose, e.g., the $\vz$ with minimum $\ell_2$-norm), and $(\vz_j)_{\comp{\S}_j} = \mathbf 0$. 
Construct the adjacency matrix $\mA$ according to $\mA = \mZ + \herm{\mZ}$, where $\mZ = \mathrm{abs}([\vz_1\, \dots \,\vz_N])$ and $\mathrm{abs}(\cdot)$ takes absolute values element-wise. 

{\bf Step 2:} Apply normalized spectral clustering \cite{ng_spectral_2001,luxburg_tutorial_2007} to $(\mA, \hat L)$. 

\vspace{0.2cm} 
{\bf The SSC algorithm:} Given a set of $N$ data points $\X$ in $\reals^p$ and an estimate of the number of subspaces $\hat L$, 
perform the following steps:

{\bf Step 1: } Let $\mX \in \reals^{p \times N}$ be the matrix whose columns are the points in $\X$. For every $\vx_j \in \X$ determine $\vz_j$ as a solution of
  \begin{align}
  \underset{\vz}{\text{minimize }} \norm[1]{\vz}  \text{ subject to }  \vx_j = \mX \vz \text{ and } z_j=0.
  \label{eq:minsscworig}
  \end{align}
  Construct the adjacency matrix $\mA$ according to $\mA = \mZ + \herm{\mZ}$, where $\mZ = \mathrm{abs} ([\vz_1\, \dots \, \vz_N])$.  

{\bf Step 2:} Apply normalized spectral clustering \cite{ng_spectral_2001,luxburg_tutorial_2007} to $(\mA, \hat L)$. 

\vspace{0.2cm} 
{\bf The SSC-OMP algorithm:}  \label{alg:sscomp}
Given a set of $N$ data points $\X$ in $\reals^p$, an estimate of the number of subspaces $\hat L$, 
and a maximum number of OMP iterations $\pit_{\max}$, perform the following steps: 

{\bf Step 1:} For every $\vx_j \in \X$, find a sparse representation of $\vx_j$ in terms of $\X \! \setminus \! \{\vx_j\}$ using OMP as follows: Initialize the iteration counter $\pit = 0$, the residual $\vr_0 = \vx_j$, and the set of selected indices $\Lambda_0 = \emptyset$. For $\pit = 1, 2, \dots$ perform updates according to
\begin{align}
\Lambda_\pit &= \Lambda_{\pit-1} \cup \underset{i \in [N] \colon i \neq j}{\mathrm{argmax}}   \left| \innerprod{\vx_i}{\vr_{\pit-1}} \right| \label{eq:OMPSelRule} \\
\vr_\pit &= (\mI - \mX_{\Lambda_\pit} \pinv{\mX}_{\Lambda_\pit} ) \vx_j \label{eq:OMPSelRuleB}
\end{align}
until $\vr_\pit = \vect{0}$ or $\pit = \pit_{\max}$ (when the maximizer in \eqref{eq:OMPSelRule} is not unique, select any of the solutions). 
With the number of OMP iterations actually performed denoted by $s_\mathrm{end}$, set $(\vz_j)_{\Lambda_{\pit_{\mathrm{end}}}} = \pinv{\mX}_{\Lambda_{\pit_{\mathrm{end}}}} \vx_j$, $(\vz_j)_{\comp{\Lambda}_{\pit_{\mathrm{end}}}} = \vect{0}$, and construct the adjacency matrix $\mA$ according to $\mA = \mZ + \herm{\mZ}$, where $\mZ = \mathrm{abs} ([\vz_1\, \ldots \,\vz_N])$.

{\bf Step 2:} Apply normalized spectral clustering \cite{ng_spectral_2001,luxburg_tutorial_2007} to $(\mA, \hat L)$. 
\vspace{0.2cm} 

For all three algorithms the number of subspaces $L$ can be estimated based on the insight that the number of zero eigenvalues of the normalized Laplacian of the graph $G$ with adjacency matrix $\mA$, henceforth simply referred to as ``the graph $G$'', is equal to the number of connected components of $G$ \cite{spielman_spectral_2012}. A robust estimator for $L$ is the \emph{eigengap heuristic} described in \cite{luxburg_tutorial_2007}.  

Let the oracle segmentation of $\X$ be given by $\X = \X_1 \cup \ldots  \cup  \X_L$.  
  If each connected component 
  in the graph $G$ 
  corresponds exclusively to points from one of the sets $\X_\l$, spectral clustering will deliver the oracle segmentation \cite[Prop.~4]{luxburg_tutorial_2007} and the clustering error, i.e., the fraction of misclustered points, will be zero. 
Since conditions guaranteeing zero clustering error are inherently hard to obtain, we will work with an intermediate, albeit sensible, performance measure, also employed in \cite{heckel_robust_2013,soltanolkotabi_geometric_2011,soltanolkotabi2014robust,dyer_greedy_2013}. Specifically, this measure, termed the no-false connections property, declares success if the graph $G$ 
has no false connections, i.e., if each $\vx_j \in \X_\l$ is connected to points in $\X_\l$ only, for all $\l$. 
Guaranteeing the absence of false connections, does, however, not guarantee 
 that the connected components of $G$ correspond to the $\X_\l$, as the points in a given set $\X_\l$ may form two (or more) distinct connected components in $G$. 
 
To counter this problem sufficiently many entries in each row/column of the adjacency matrix $\mA$ have to be non-zero. Specifically, for the subgraphs of $G$ corresponding to the $\X_\l$ to be connected, each row/column of $\mA$ corresponding to a point in $\X_\ell$ needs to have between $O(\log n_\l)$ and $O(n_\l)$ non-zero entries. 
As the solutions $\vz$ to $\arg \min_{\vz} \norm[2]{\vx_j - \mX_{\S_j} \vz  }$ are typically dense, TSC is likely to select a representation of $\vx_j$ in terms of points in $\X_\l \! \setminus \! \{ \vx_j \}$ with on the order of $q$ non-zero coefficients. 
Choosing $q$ large enough therefore ensures sufficient connectivity of the graph $G$ generated by TSC. 
On the other hand, taking $q$ to be large increases the probability of false connections. The performance guarantee we obtain for TSC therefore requires $q$ to be sufficiently small relative to the $n_\l$. 

For SSC and SSC-OMP, the number of non-zero entries in each row/column of $\mA$ turns out to be tied to $d_\ell$, rather than $n_\ell$. 
To see this, suppose that both algorithms exclusively select data points from $\X_\l \! \setminus \! \{\vx_j\}$ to represent $\vx_j$. Moreover, assume that the $\X_\l$ are non-degenerate in the sense that, indeed, $d_\ell$ points are needed to represent $\vx_j \in \X_\ell$ through points in $\X_\l \! \setminus \! \{\vx_j\}$; this precludes, e.g., that $\X_\l$ contains multiple copies of the same data point. 
The OMP algorithm in SSC-OMP then terminates after $\min(d_\l, \pit_{\max})$ (recall that $d_\l = \mathrm{dim}(\cS_\l)$) iterations for $\vx_j \in \X_\ell$ and hence results in exactly $\min(d_\l, \pit_{\max})$ non-zero entries in the corresponding column of $\mZ$ (recall that $\mA = \mZ + \herm{\mZ}$). 
For SSC, we simply note that $d_\ell$ points are enough to represent $\vx_j \in \X_\ell$ through other points in $\X_\l$ and we cannot guarantee more than $d_\ell$ non-zero entries in the corresponding column of $\mZ$, in general. 
This will lead to insufficient connectivity for SSC and SSC-OMP when $d_\l$ is not in the range $O(\log n_\l)$--$O(n_\l)$. The problem is exacerbated when the data set is degenerate.
To counter insufficient connectivity in SSC a modification which adds an $\ell_2$-penalty to the cost function in \eqref{eq:minsscworig} was proposed in \cite[Sec.~5]{elhamifar_sparse_2013}. Such a modification is not known for SSC-OMP, and this may be considered a limitation of SSC-OMP.

We finally remark that TSC and SSC-OMP can be made essentially parameterless, like SSC. Specifically, a procedure for choosing the TSC parameter $q$ in a data-driven fashion is described in \cite{heckel_neighborhood_2014}, 
and for SSC-OMP we can get rid of the parameter $s_{\max}$ by stopping the OMP step once the $\ell_2$-norm of the residual $\vr_\pit$ falls below a threshold value.

\section{\label{sec:perfguar} Main results}

We start by specifying the statistical data model used throughout the paper. 
The subspaces $\cS_\l$ are taken to be deterministic and the points within the $\cS_\l$ are chosen randomly. Specifically, the elements of the set $\YS_\l$ in $\YS = \YS_1 \cup \ldots  \cup  \YS_L$ are 
 obtained by choosing $n_\l$ points at random according to $\vy_j^{(\l)} = \mU^{(\l)} \va^{(\l)}_j, j \in [n_\l]$, where the columns of $\mU^{(\l)} \in \reals^{m\times d_\l}$ form an orthonormal basis for the $d_\l$-dimensional subspace $\cS_\l$, and the $\va^{(\l)}_j$ are i.i.d.~uniform on $\US{d_\l}$. As the $\mU^{(\l)}$ are orthonormal, the data points $\vy_j^{(\l)}$ are distributed uniformly  on the set $\{\vy \in \cS_\l \colon \norm[2]{\vy} = 1 \} = \cS_\l \cap \US{m}$, which avoids degenerate situations where the data points lie in preferred directions. 
To see why such degeneracies can lead to ambiguous results, consider a two-dimensional subspace and assume that the data points in this subspace are skewed towards two distinct directions. Then, there are two sensible segmentations. One is to assign the points corresponding to each direction to separate clusters, the other to assign all points to one cluster. 

The dimensionality-reduced data set $\X \subset \reals^p$ is obtained by applying the (same) realization of a random matrix $\mPh \in \reals^{p\times m}$ ($p \geq \max_{\l} d_\l$) to each point in $\YS$. 
The elements of the sets $\X_\l$ in $\X = \X_1 \cup \ldots \cup \X_L$ are hence given by $\vx_j^{(\l)} = \mathbf{\Phi} \vy_j^{(\l)}, j \in [n_\l]$. 
We take $\mPh$ as a random matrix satisfying the following concentration inequality 
\begin{align}
\PR{ \left| \norm[2]{\mPh \vx}^2 - \norm[2]{\vx}^2 \right|  \geq t \norm[2]{\vx}^2 } \leq 2 e^{- \tilde c t^2 p}, \quad \forall  \, t>0, \forall \vx \in \reals^{m},
\label{eq:conceqcondonPh}
\end{align}
where $\tilde c$ is either a numerical constant or a parameter mildly depending on $m$. 
Random matrices satisfying \eqref{eq:conceqcondonPh} realize, with high probability, linear embeddings in the sense of the Johnson-Lindenstrauss (JL) Lemma, see e.g., \cite{vempala_random_2005}, \cite[Sec. 9.5]{foucart_mathematical_2013}. 
The JL Lemma  says that every set of $N$ points in Euclidean space can be embedded in an $O(\epsilon^{-2} \log N)$-dimensional space without perturbing the pairwise Euclidean distances between the points by more than a factor of $1 \pm \epsilon$. 

A similar statement on random projection preserving affinities between subspaces--as defined in \eqref{eq:affdef}--is used in our proofs. Specifically, we show that 
randomly projecting a set of $d$-dimensional subspaces into $p$-dimensional space does not increase their pairwise affinities by more than const.$\sqrt{d/p}$, with high probability (cf.~\eqref{eq:prass2}).  
The concentration inequality \eqref{eq:conceqcondonPh} holds, inter alia, for matrices with i.i.d.~subgaussian\footnote{A random variable $x$ is subgaussian \cite[Sec.~7.4]{foucart_mathematical_2013} if its tail probability satisfies $\PR{|x| >t} \leq c_1 e^{-c_2 t^2}$ for constants $c_1, c_2>0$. 
}
entries \cite[Lem.~9.8]{foucart_mathematical_2013}; this includes $\mathcal N(0,1/p)$ entries and entries that are uniformly distributed on $\{-1/\sqrt{p}, 1/\sqrt{p} \}$. 
Such matrices may, however, be costly to generate, store, and apply to high-dimensional data points. In order to reduce these costs structured random matrices satisfying \eqref{eq:conceqcondonPh} (with $\tilde c$ possibly mildly dependent on $m$) were proposed in \cite{ailon_almost_2013,krahmer_new_2011}. 
For example, the structured random matrix proposed in \cite{ailon_almost_2013} (and described in detail in Section \ref{sec:numres}) satisfies \eqref{eq:conceqcondonPh} with $\tilde c = c_2 \log^{-4}(m)$, where $c_2$ is a numerical constant \cite[Prop.~3.2]{krahmer_new_2011}, and can be applied in time $O( m \log m)$ as opposed to time $O(mp)$ for the realizations of general subgaussian random matrices.

The clustering performance guarantees we obtain below  
 are all in terms of the affinity between the subspaces $\cS_k$ and $\cS_\l$ defined as 
\cite[Def.~2.6]{soltanolkotabi_geometric_2011}, \cite[Def.~1.2]{soltanolkotabi2014robust} 
\begin{align} 
\aff(\cS_k,\cS_\l) \defeq \frac{1}{\sqrt{d_k \land d_\l }}    \big\| \herm{\mU^{(k)}} \mU^{(\l)}  \big\|_F. 
\label{eq:affdef}
\end{align}
Note that $0 \leq \aff(\cS_k,\cS_\l) \leq 1$, with $\aff(\cS_k,\cS_\l) = 1$ if $\cS_k \subseteq \cS_\l$ or $\cS_\l \subseteq \cS_k$ and $\aff(\cS_k,\cS_\l) = 0$ if $\cS_k$ and $\cS_\l$ are orthogonal to each other. 
Moreover, we have
\begin{equation}
\aff(\cS_k,\cS_\l) =\allowbreak \sqrt{ \cos^2( \theta_1) + \ldots + \cos^2(\theta_{d_k \land d_\l})} \allowbreak /\sqrt{d_k \land d_\l}, \label{eq:princangles}
\end{equation}
where $\theta_1 \leq \ldots \leq \theta_{d_k \land d_\l }$ are the principal angles between $\cS_k$ and $\cS_\l$ \cite[Sec.~6.3.4]{golub_matrix_1996}. 
If $\cS_k$ and $\cS_\l$ intersect in $t$ dimensions, i.e., if $\cS_k \cap \cS_\l$ is $t$-dimensional, then $\cos(\theta_1)= \ldots =\cos(\theta_{t})=1$ and hence $\aff(\cS_k,\cS_\l) \geq \sqrt{t/ (d_k \land d_\l) }$. 
The affinity between subspaces plays an important role in subspace classification \cite{nokleby_discrimination_2015} as well, see \cite[Thms.~2 and 3]{huang_role_2015}. 

We start with our main result for TSC. 

\begin{theorem}
\label{thm:tscrp}
Choose $q$ such that $q \leq \min_\l n_\l /6$. If
\begin{align}
\max_{k,\l \colon  k\neq \l}
\aff(\cS_k,\cS_\l)  
+ \frac{\sqrt{11}}{\sqrt{3 \tilde c}}  \frac{\sqrt{d_{\max}}}{\sqrt{p} } \leq \frac{1}{15 \log N },
\label{eq:adikqopol}
\end{align}
where $d_{\max} = \max_\l d_\l$ and $\tilde c$ is the constant in the concentration inequality \eqref{eq:conceqcondonPh},  then the graph $G$ obtained by applying TSC to $\X$ has no false connections with probability at least $1 - 7 N^{-1} - \sum_{\l=1}^L n_\l e^{-c(n_\l-1)}$, where $c>1/20$ is a numerical constant. 
\end{theorem}

Our main result for SSC is the following. 

\begin{theorem}
\label{th:RPNoiseless}
Let $\rho_\l \defeq (n_\l-1)/d_\l$, $\l \in [L]$, $\rho_{\min} \defeq \min_\l \rho_\l \geq \rho_0$, where $\rho_0 > 1$ is a numerical constant, and pick any $\tau>0$. Set $d_{\max} = \max_\l d_\l$ and suppose that 
\begin{align}
\label{eq:ThmAffConditionrp}
\max_{k,\l \colon  k\neq \l} 
\mathrm{aff}(\cS_k,\cS_\l) 
\!+\! \sqrt{\frac{28 d_{\max} \!+\! 8 \log L \!+\! 2\tau }{3 \tilde cp}}
\leq \frac{ \sqrt{\log \rho_{\min}} }{65 \log N  },
\end{align}
where $\tilde c$ is the constant in \eqref{eq:conceqcondonPh}. 
Then, the graph $G$ obtained by applying  SSC to $\X$ has no false connections with probability at least 
$1 - 4e^{-\tau/2} - N^{-1}  -  \sum_{\l=1}^L n_\l e^{-\sqrt{\rho_\l}  d_\l}$. 
\end{theorem}

Finally, for SSC-OMP we obtain the following statement.

\begin{theorem}
\label{th:OMPRPnoiseless}
Let $\rho_\l \defeq (n_\l-1)/d_\l$, $\l \in [L]$, $\rho_{\min} \defeq \min_\l \rho_\l \geq \rho_0$, where $\rho_0 > 1$ is a numerical constant, and pick any $\tau>0$. 
Set $d_{\min} \defeq \min_\l d_\l$, $d_{\max} \defeq \max_\l d_\l$, and suppose that $\Ph$ has $($in addition to satisfying the concentration inequality \eqref{eq:conceqcondonPh}$)$ a rotationally invariant distribution, i.e., $\Ph \mV \sim \Ph$ for all unitary matrices $\mV \in \reals^{m \times m}$. 
If
\begin{align}
\label{eq:ThmAffConditionOMPSSCRP}
\max_{k,\l  \colon  k\neq \l} 
\aff(\cS_k,\cS_\l) 
+
\sqrt{ \frac{28d_{\max} + 8 \log L + 2\tau}{12 \tilde c p} } 
\sqrt{\frac{d_{\max}}{d_{\min}}}
\leq
\frac{3}{200} \frac{\sqrt{\log \rho_{\min}}}{\log N} ,
\end{align}
where $\tilde c$ is the constant in \eqref{eq:conceqcondonPh}, then,
 irrespectively\footnote{
 While the statement holds irrespectively of $s_{\max}$, recall from Section \ref{sec:AlgIntro} that choosing $s_{\max}$ too small may result in 
too few non-zeros in the adjacency matrix $\mA$ for successful clustering.} of the choice of the maximum number of OMP-iterations $s_{\max}$, the graph $G$ obtained by applying  SSC-OMP to $\X$ has no false connections with probability at least 
$1 - 4e^{-\tau/2} - 4N^{-1}  - \sum_{\l=1}^L n_\l e^{-\sqrt{\rho_\l}  d_\l}$.  
\end{theorem}

The proofs of Theorems \ref{thm:tscrp}, \ref{th:RPNoiseless}, and \ref{th:OMPRPnoiseless} are provided in Appendices \ref{sec:proofRPTSC}, \ref{sec:proofRPBP}, and \ref{sec:proofOMPSSC}, respectively, and are established by first deriving deterministic clustering conditions that are then evaluated for our statistical data model. 

Theorems \ref{thm:tscrp} and \ref{th:RPNoiseless} essentially say that even when  $p$ is on the order of $d_{\max}$, TSC and SSC succeed with high probability if the affinities between the subspaces $\cS_\l$ are sufficiently small and if $\X$ contains sufficiently many points from each subspace. 
The same conclusion applies to SSC-OMP provided that the 
term $\sqrt{d_{\max} / d_{\min}}$ is not too large, which is the case if the dimensions $d_\l$, $\l \in [L]$, of the subspaces are of the same order. 
This condition is satisfied in many practical applications, such as, e.g., for the face clustering and the handwritten digit clustering problems described in Section \ref{sec:numres}. 
We believe the occurrence of the factor $\sqrt{d_{\max}/d_{\min}}$ in \eqref{eq:ThmAffConditionOMPSSCRP} to be an artifact of our proof technique. 
Also note that Theorem \ref{th:OMPRPnoiseless} 
imposes more restrictive conditions on $\mPh$ than Theorems \ref{thm:tscrp} and \ref{th:RPNoiseless}, namely the distribution of $\Ph$ has to be rotationally invariant. 
This is a technical condition and it is not implied by \eqref{eq:conceqcondonPh}. Examples of rotationally invariant matrices satisfying  \eqref{eq:conceqcondonPh} include matrices with i.i.d. $\mathcal N(0,1/p)$ entries. 

Theorems \ref{thm:tscrp}--\ref{th:OMPRPnoiseless} apply even when the subspaces $\cS_\l$ span the ambient space $\reals^m$. 
This follows by virtue of our clustering conditions depending only on the pairwise affinities between subspaces, and pairwise affinities changing only moderately if the dimensionality is reduced down to no more than the order of the individual subspace dimensions. 

Theorems \ref{thm:tscrp}--\ref{th:OMPRPnoiseless} show that for all three algorithms, $p$ may be taken to be linear (up to $\log$-factors) in $d_{\max}$. 
We can therefore conclude that the dimensionality of the data set $\YS$ can be reduced down to the order of the largest subspace dimension without affecting  clustering performance significantly. 
This has important practical ramifications as, for all three algorithms considered, the computational complexity associated with the construction of the adjacency matrix is essentially linear in the dimension of the ambient space the data points ``live in''. 
To get an idea of the resulting overall complexity savings, let us consider the TSC algorithm 
and assume that the (high-dimensional) data set $\YS \subset \reals^m$ is projected down to $\reals^p$, with $p = O( d_{\max} \log^2(N) )$, via a Gaussian random projection; this choice of $p$ guarantees,  
by Theorem \ref{thm:tscrp}, that clustering performance is not affected significantly by dimensionality reduction. 
The complexity associated with the construction of the adjacency matrix for TSC 
is given by the cost of computing the inner products between all pairs of data points, 
and is therefore $O(m N^2)$ for the original data set $\YS \subset \reals^m$ and $O(p N^2)$ for the projected data set $\X \subset \reals^p$. 
Adding the cost for applying the Gaussian random projection results in an overall cost of 
$O( p N^2) + O(p N m) = O( d_{\max} \log^2(N)  N( N+ m))$ for building the adjacency matrix associated with $\X$. 
The resulting complexity savings for TSC are therefore given by $O( \min(m,N)/ (d_{\max} \log^2(N))$.   
The absolute run-time savings are even more pronounced for SSC-OMP and SSC, as the corresponding costs for building the adjacency matrix is larger than $O(m N^2)$. 
Further gains can be obtained by employing fast random projections \cite{ailon_almost_2013}. 

Dimensionality reduction affects the computational cost associated with the construction of the adjacency matrix only. 
The spectral clustering step, which when na\"ively implemented has complexity $O(N^3)$, may be the dominating factor in the overall computational cost, in particular when $m$ is small relative to $N^3$. 
Notwithstanding, dimensionality reduction can still lead to significant total run-time savings. Our numerical results in Section \ref{sec:numres} demonstrate this for SSC. 
To see savings on the same order for SSC-OMP and TSC, we
would have to consider problems with $N$ smaller relative to $m$. 

The probability lower bounds in Theorems \ref{thm:tscrp}--\ref{th:OMPRPnoiseless} are independent of $p$ and $m$ and require the total number of data points $N$ to be large in absolute terms in order to ensure a success probability close to one. 

Theorems \ref{thm:tscrp}--\ref{th:OMPRPnoiseless} are order-optimal in the following sense. 
If dimensionality is reduced to below $d_{\max}$, then, in general, there are points from different subspaces that are projected into the same lower-dimensional subspace, which renders the resulting clustering problem fundamentally ill-posed. 
To see this, take $d_\ell=d$, for all $\ell$, and assume that $p\leq d$. 
Next, note that the (randomly projected) points $\X_\l$ lie in the column span of $\mPh \mU^{(\ell)}$. As $\mU^{(\ell)}$ is a basis for the $d_\l$-dimensional subspace $\cS_\l \subset \reals^m$, the span of $\mPh \mU^{(\ell)}$ is $\reals^p$, for all $\ell$, and therefore all points in the projected data set $\X = \X_1 \cup ... \cup \X_L$ lie in the same $p$-dimensional subspace, which renders the clustering problem ill-posed. 


We next compare the clustering conditions \eqref{eq:adikqopol}, \eqref{eq:ThmAffConditionrp}, and \eqref{eq:ThmAffConditionOMPSSCRP} in Theorems \ref{thm:tscrp}, \ref{th:RPNoiseless}, and \ref{th:OMPRPnoiseless} with their counterparts for clustering of the original, high-dimensional data set $\YS$. Specifically, such reference conditions can be found in \cite[Thm.~2]{heckel_robust_2013} for TSC and in \cite[Thm.~2.8]{soltanolkotabi_geometric_2011} for SSC, 
but do not seem to be available for SSC-OMP for the statistical data model considered in this paper. 
However, setting $\Ph = \mI$ in the proof of Theorem \ref{th:OMPRPnoiseless}, we can easily get a reference condition for SSC-OMP. 
Rather than providing the details of this simple modification, we refer the reader to the proof in \cite[Chap.~4]{tschannen_sparse_2014}. 

\begin{corollary} \label{cor:ompssc}
Let $\rho_\l \defeq (n_\l-1)/d_\l$, $\l \in [L]$, and suppose that $\rho_{\min} \defeq \min_\l \rho_\l \geq \rho_0$, where $\rho_0 > 1$ is a numerical constant. 
If
\begin{align}
\max_{k,\l  \colon  k\neq \l} 
\aff(\cS_k,\cS_\l) 
\leq
\frac{\sqrt{\log \rho_{\min}}}{64 \log N},  \label{eq:ssompunprojcond}
\end{align}
then the graph $G$ obtained by applying SSC-OMP to the original, high-dimensional data set $\YS$ has no false connections with probability at least 
$1 - 2N^{-1}  - \sum_{\l=1}^L n_\l e^{-\sqrt{\rho_\l}  d_\l}$.
\end{corollary}

We conclude that for all three algorithms the impact of dimensionality reduction is essentially quantified through a term proportional to $\sqrt{d_{\max}/p}$ that adds to the maximum affinity between the subspaces $\cS_\l$ in the clustering conditions \eqref{eq:adikqopol}, \eqref{eq:ThmAffConditionrp}, and \eqref{eq:ThmAffConditionOMPSSCRP}. 
These clustering conditions nicely reflect the intuition that the smaller the affinities between the subspaces $\cS_\l$, the more aggressively we can reduce the dimensionality of the data set without compromising clustering performance. 

As the result in Corollary \ref{cor:ompssc} is new, a few comments on its relation to existing results, specifically those in \cite{dyer_greedy_2013} and \cite{you_sparse_2015}, are in order. 
Corollary \ref{cor:ompssc} imposes less restrictive conditions on the relative orientations of the subspaces than \cite[Thm.~3]{dyer_greedy_2013}, \cite[Thm.~2, Cor.~1]{you_sparse_2015}, but makes stronger assumptions on the data model. 
The result in \cite[Thm. 3]{you_sparse_2015} applies to subspaces with random orientations, and therefore does not allow for statements involving subspace affinities.
We refer the reader to the thesis \cite[Sec.~4.1]{tschannen_sparse_2014} for a more detailed comparison of Corollary \ref{cor:ompssc} above to \cite[Thm.~3]{dyer_greedy_2013}. Finally, numerical results corroborating the fundamental nature of the clustering condition \eqref{eq:ssompunprojcond} can be found in \cite[Sec.~5.1]{tschannen_sparse_2014}.


\section{Impact of noise}

In many practical applications the data points to be clustered are corrupted by noise, typically modeled as additive Gaussian noise. In this section, we 
study the interplay between dimensionality reduction and additive noise for the TSC algorithm. 
Specifically, we let the high-dimensional data points be corrupted by Gaussian noise according to
\[
\tilde \vy_i^{(\l)} = \vy_i^{(\l)} + \ve_i^{(\l)},
\]
where $\ve_i^{(\l)} \sim \mc N(0,(\sigma^2/m) \mI)$, and assume, as before, that $\vy_i^{(\l)}$ is drawn i.i.d.~uniformly from the intersection of the $d_\ell$-dimensional subspace $\cS_\l$ with the unit sphere. 
The dimensionality-reduced noisy data set $\tilde \X \subset \reals^p$ is obtained by applying the same realization of the random projection matrix $\mPh \in \reals^{p\times m}$ to all (noisy) data points $\tilde \vy_i^{(\l)}$. 
The elements of the sets $\tilde \X_\l$ in $\tilde \X = \tilde \X_1 \cup \ldots \cup \tilde \X_L$ are hence given by 
\begin{align}
\tilde \vx_j^{(\l)} = \mathbf{\Phi} (  \vy_i^{(\l)} + \ve_i^{(\l)}  ), \quad j \in [n_\l]. 
\label{eq:noisemodel}
\end{align}

\begin{theorem}
\label{thm:tscrpnoisy}
Choose $q$ such that $q \leq \min_\l n_\l /6$, and let $ m\geq 6\log N$. If
\begin{align}
\max_{k,\l \colon  k\neq \l}
\aff(\cS_k,\cS_\l)  
+ 
\frac{\sqrt{11}}{\sqrt{3 \tilde c}}  \frac{\sqrt{d_{\max}}}{\sqrt{p} } 
+ \frac{\sigma (1+\sigma)\sqrt{6} }{\sqrt{ \bar c \log N}} \frac{\sqrt{d_{\max}} }{\sqrt{p} }
\leq \frac{1}{15 \log N },
\label{eq:adikqopolnoisy}
\end{align}
where $d_{\max} = \max_\l d_\l$ and 
$\bar c = \min(6,\tilde c)$ with $\tilde c$  
 the constant in the concentration inequality \eqref{eq:conceqcondonPh},  then the graph $G$ obtained by applying TSC to $\tilde \X$ has no false connections with probability at least $1 - 14 N^{-1} - 2Ne^{-m}  - \sum_{\l=1}^L n_\l e^{-c(n_\l-1)}$, where $c>1/20$ is a numerical constant. 
\end{theorem}


Theorem \ref{thm:tscrpnoisy} states that in the noisy case---just as in the noiseless case---TSC succeeds for $p$ as small as $d_{\max}$, order-wise, provided that the affinities between the subspaces $\cS_\l$ are sufficiently small and $\tilde \X$ contains sufficiently many points from each subspace. 
More specifically, comparing the noiseless clustering condition \eqref{eq:adikqopol} to \eqref{eq:adikqopolnoisy}, we can see that the impact of noise is simply to add the offset $\frac{\sigma (1+\sigma)\sqrt{6} }{\sqrt{\bar c \log N}} \frac{\sqrt{d_{\max}} }{\sqrt{p} }$ to the LHS of the clustering condition. 
For fixed $\sigma$, owing to the factor $\sqrt{d_{\max}/p}$, the impact of noise on the effective affinity as quantified by the LHS of \eqref{eq:adikqopolnoisy} becomes more pronounced when the dimensionality is reduced more aggressively.

Theorem \ref{thm:tscrpnoisy} continues to hold (with $\bar c$ in the term $\frac{\sigma (1+\sigma)\sqrt{6} }{\sqrt{ \bar c \log N}} \frac{\sqrt{d_{\max}} }{\sqrt{p} }$ replaced by a numerical constant, and $e^{-m}$ in the success probability replaced by $e^{-m}$), if noise $\tilde \ve_i^{(\l)} \sim \mc N(0,(\sigma^2/p) \mI)$ 
is added \emph{after} random projection according to $\tilde \vx_j^{(\l)} = \mathbf{\Phi}   \vy_i^{(\l)} + \tilde \ve_i^{(\l)} $. 
This is not surprising, as the absolute amount of noise injected remains the same, i.e., $\EX{\norm[2]{ \tilde \ve_i^{(\l)} }^2 } = \EX{\norm[2]{ \ve_i^{(\l)} }^2 } = \sigma^2$. 

We finally note that an approach similar to that used for TSC can be applied to extend our result for SSC-OMP to the noisy case resulting in 
 clustering conditions analogous to those for TSC.  
The corresponding technical details are, however, significantly more involved and cumbersome. 
We therefore decided not to state the formal result. 
We expect that a similar result can be proven for (a robust version of) SSC as our simulation  results in Section \ref{sec:impnoise} indicate that  the qualitative behavior of all three algorithms in the presence of noise is essentially identical, and, in addition, is qualitatively accurately predicted by Theorem \ref{thm:tscrpnoisy}. 

\newcommand{\p}{p} 

\section{Numerical Results} \label{sec:numres}

%
%
%
%

We evaluate the impact of dimensionality reduction on
the clustering error (CE), i.e., the fraction of misclustered points, for TSC, SSC, and SSC-OMP applied to synthetic data as well as to publicly available standard data sets widely used in the subspace clustering literature. Specifically, we consider the problems of clustering faces, handwritten digits, and gene expression data. 
All three algorithms, TSC, SSC, and SSC-OMP, were observed to tolerate massive dimensionality reduction in all experiments. 
The performance ranking of the three algorithms  according to CE varies  considerably across data sets. Specifically, in order to demonstrate that none of the algorithms uniformly outperforms the others, we chose to report the results for all three data sets. 
We also compare the algorithms in terms of their running times on a PC with 32 GB RAM and 8-core Intel Core i7-3770K CPU clocked at 3.50 GHz. 

TSC and SSC-OMP were implemented in Matlab following the specifications in Section \ref{sec:AlgIntro}. For SSC, we used the Matlab implementation provided in \cite{elhamifar_sparse_2013}, which is based on Lasso (instead of $\ell_1$-minimization) and uses the Alternating Direction Method of Multipliers (ADMM). Code to reproduce the experiments in this section is available at \url{http://www.nari.ee.ethz.ch/commth/research/}. 
Information on the number of Monte Carlo runs used in our experiments is contained in this Matlab code. 

Unless stated otherwise, we select the Lasso parameter $\lambda$ in SSC from the set $\{0.001, 0.002,\allowbreak 0.004, 0.008, 0.01, 0.02, 0.04, 0.08, 0.1, 0.2\}$ such that the lowest clustering error is obtained on the original high-dimensional data set $\YS$. The parameters $q$ and $s_{\max}$ for TSC and SSC-OMP, respectively, are chosen analogously from the set $\{2, 4, \dots, 18\}$. Although these parameter selection procedures may not yield the optimum parameters for the projected data set $\X$ for all realizations of $\Ph$, we desist from selecting the parameters for every realization of $\Ph$ individually as this may lead to overly optimistic results.

As projection matrices we consider i.i.d.~$\mathcal N(0,1/p)$ Gaussian random matrices (referred to as GRP) and fast random projection (FRP) matrices \cite{ailon_almost_2013} 
given by the real part of $\mF \mD \in \complexset^{p \times m}$, where $\mD \in \reals^{m\times m}$ is diagonal with main diagonal elements drawn i.i.d.~uniformly from $\{-1,1\}$, and $\mF \in \complexset^{p\times m}$ is obtained by choosing a set of $p$ rows uniformly at random from the rows of an $m \times m$ discrete Fourier transform (DFT) matrix. 
In all experiments the dimensionality-reduced data set $\X$ is obtained 
by applying the (same) realization of either a GRP or an FRP matrix to all data points in $\YS$.  
The FRP can be implemented efficiently by premultiplying $\YS$ by $\mD$ and then applying the FFT to each data point. With regards to storage space, we note that the FRP only requires the storage of a binary $m$-dimensional vector (namely the diagonal entries of $\mD$), in contrast to $mp$ real numbers for GRPs.

\newcommand{\plotred}{red!90!white}
\newcommand{\plotgreen}{brown}
\newcommand{\plotblue}{blue!70!black}

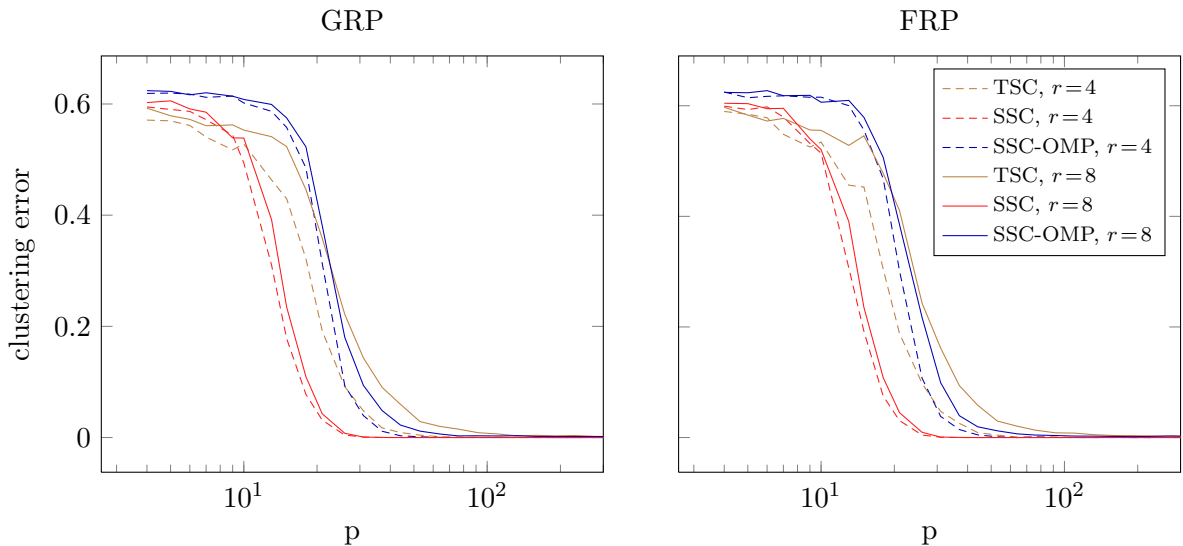
\begin{figure}[h!]

    \centering
\begin{tikzpicture}[scale=1] 

    \begin{groupplot}[%
      group style={%
        group size=2 by 1,
        xlabels at=edge bottom, 
      ylabels at=edge left,
      xticklabels at=edge bottom,
      yticklabels at=edge left,
      },
      xmax=300,
      xlabel = p,
      ylabel = {clustering error},
      width = 0.5\textwidth
      ]


\nextgroupplot[xmode=log, title=GRP]%
    \addplot +[mark=none,densely dashed,\plotgreen] table[x index=0,y index=2]{./data/CESynth0.2TSCb.dat};
    \addplot +[mark=none,densely dashed,\plotred] table[x index=0,y index=1]{./data/CESynth0.2TSCb.dat};
    \addplot +[mark=none,densely dashed,\plotblue] table[x index=0,y index=3]{./data/CESynth0.2TSCb.dat};
    \addplot +[mark=none,solid,\plotgreen] table[x index=0,y index=2]{./data/CESynth0.4TSCb.dat};
    \addplot +[mark=none,solid,\plotred] table[x index=0,y index=1]{./data/CESynth0.4TSCb.dat};
    \addplot +[mark=none,solid,\plotblue] table[x index=0,y index=3]{./data/CESynth0.4TSCb.dat};

\nextgroupplot[xmode=log, title={FRP},
 legend entries={ {TSC, $r\! =\! 4$}, {SSC, $r\!=\! 4$}, {SSC-OMP, $r\!=\! 4$}, {TSC, $r\! = \! 8$}, {SSC, $r\!=\!8$},  {SSC-OMP, $r\!=\!8$}}, 
    legend style={
            cells={anchor=west},
                    legend pos= north east,
                    font=\scriptsize,}
]%
    \addplot +[mark=none,densely dashed,\plotgreen] table[x index=0,y index=5]{./data/CESynth0.2TSCb.dat};
    \addplot +[mark=none,densely dashed,\plotred] table[x index=0,y index=4]{./data/CESynth0.2TSCb.dat};
    \addplot +[mark=none,densely dashed,\plotblue] table[x index=0,y index=6]{./data/CESynth0.2TSCb.dat};
    \addplot +[mark=none,solid,\plotgreen] table[x index=0,y index=5]{./data/CESynth0.4TSCb.dat};
    \addplot +[mark=none,solid,\plotred] table[x index=0,y index=4]{./data/CESynth0.4TSCb.dat};
    \addplot +[mark=none,solid,\plotblue] table[x index=0,y index=6]{./data/CESynth0.4TSCb.dat};
    
\end{groupplot}    
    
\end{tikzpicture}
\caption[CE for clustering synthetic data]{\label{fig:intersectrpce} Clustering error for synthetic data
 as a function of $p$ using GRP (left) and FRP (right). Recall that for $r = 4$ and $r = 8$ we have $\aff(\cS_k,\cS_\l) \geq \sqrt{1/5}$ and $\aff(\cS_k,\cS_\l) \geq \sqrt{2/5}$, respectively, for all $k,\l \in [L]$, $k \neq \l$.}
\end{figure}
    
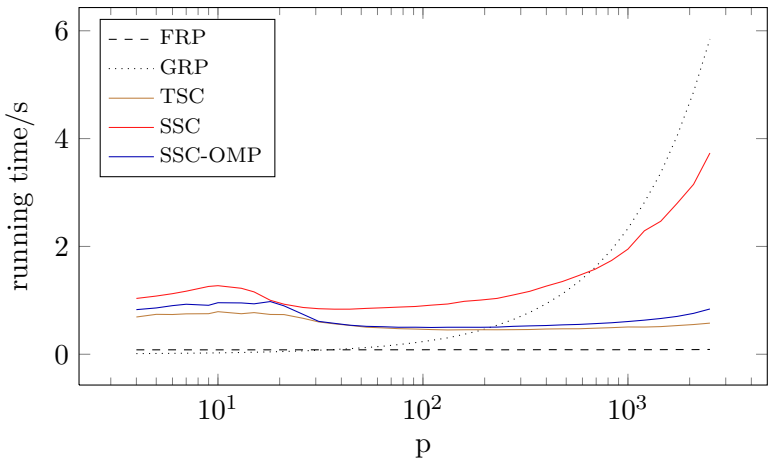
\begin{figure}[h!]
\centering
\begin{tikzpicture}[scale=1] 
    \begin{semilogxaxis}[name=plot3,
            xlabel=p,
    ylabel = {running time/s},
    width=0.65\textwidth,
    height = 0.4\textwidth,
    legend entries={ {FRP}, {GRP}, {TSC}, {SSC}, {SSC-OMP}},
    legend style={
            cells={anchor=west},
                    legend pos= north west,
                    font=\scriptsize,}
    ]
    \addplot +[mark=none,dashed,black] table[x index=0,y index=2]{./data/TCompSynth0.4TSCb.dat};
    \addplot +[mark=none,dotted,black] table[x index=0,y index=1]{./data/TCompSynth0.4TSCb.dat};
    \addplot +[mark=none,solid,\plotgreen] table[x index=0,y index=4]{./data/TCompSynth0.4TSCb.dat};
    \addplot +[mark=none,solid,\plotred] table[x index=0,y index=3]{./data/TCompSynth0.4TSCb.dat};
    \addplot +[mark=none,solid,\plotblue] table[x index=0,y index=5]{./data/TCompSynth0.4TSCb.dat};
    \end{semilogxaxis}
    
\end{tikzpicture}
\caption[Running times for clustering synthetic data]{\label{fig:intersectrprt} Running times (in seconds) for clustering synthetic data.}
\end{figure}

\subsection{Synthetic data}

\subsubsection{Comparison of TSC, SSC, and SSC-OMP}

We use the data model described in Section \ref{sec:perfguar} with $m=2^{15}=32768$ and generate $L=3$ subspaces $\cS_\l$ of $\reals^{m}$ of dimension $d=20$ at random such that every pair of subspaces intersects in at least $r$ dimensions; this implies $\aff(\cS_k,\cS_\l) \geq \sqrt{r/d}$, for all $k,\l \in [L], k\neq \l$. More specifically, we take the basis matrices to be given by $\mU^{(\l)} = [\mU \;\, \tilde \mU^{(\l)}]$, where $\mU \in \reals^{m \times r}$ and the $\tilde \mU^{(\l)} \in \reals^{m \times (d-r)}$, $\l \in [L]$, are chosen uniformly at random among all orthonormal matrices of dimensions $m \times r$ and $m \times (d-r)$, respectively.
We sample $n_\l = 80$ data points, for each $\l \in [L]$, resulting in a total of $N=240$ data points.

In Figure \ref{fig:intersectrpce}, we plot the CE as a function of $p$ for TSC, SSC, and SSC-OMP applied to the dimensionality-reduced data set $\X$ with $r=4$ and $r=8$. Figure \ref{fig:intersectrprt} shows the running times corresponding to the application of the FRP and the GRP matrix to the (entire) data set $\YS$ along with the running times of the clustering algorithms alone. 

The results show, as predicted by Theorems  \ref{thm:tscrp}--\ref{th:OMPRPnoiseless}, 
 that TSC, SSC, and SSC-OMP, indeed, succeed provided that $\sqrt{d/p}$ is sufficiently small. 
Specifically, we observe a transition to $\text{CE} \approx 0$ for $p$ between $20$ and $100$. 
As the subspaces $\cS_\l$ are of dimension $20$ this corroborates the fact that the dimensionality of the data can be reduced down to the dimension of the subspaces without compromising clustering performance significantly. Equivalently, we accomplish a dimensionality reduction by a factor of about $1600$--$320$.

For all three algorithms the numerical results further confirm the tradeoff between the affinities of the $\cS_\l$ and the amount of dimensionality reduction possible as quantified by the clustering conditions  \eqref{eq:adikqopol}, \eqref{eq:ThmAffConditionrp}, and \eqref{eq:ThmAffConditionOMPSSCRP}. 
Specifically, the CE increases as $r$ and hence $\aff(\cS_k,\cS_\l)$ increases.
In this example, 
SSC consistently outperforms TSC and SSC-OMP, albeit at the cost of significantly longer running time (see Figure \ref{fig:intersectrprt}). While the running time of SSC exhibits very pronounced increasing behavior in $p$, that of SSC-OMP shows much less pronounced increases, and that of TSC does not increase notably in $p$.
It is furthermore interesting to see that the clustering performance is essentially identical for FRP and GRP. This is remarkable as the application of FRP requires only $O(m \log m)$ operations (per data point) and therefore its running time does not depend on $p$. Application of the GRP, in contrast, requires $O(mp)$ operations (per data point), which results in a running time that is linear in $p$. 

%

\subsubsection{\label{sec:impnoise}Impact of noise}

In the next experiment we study the interplay between noise and dimensionality reduction.  
We use the data model described in Section \ref{sec:perfguar} with $m=100$ and generate $L=2$ orthogonal subspaces $\cS_\l$ of $\reals^{m}$ of dimension $d=10$. This ensures that the affinity between the subspaces equals $0$ (fixing the affinity to some other constant would not change the qualitative conclusions).
We generate the noisy data set $\tilde \YS$ by sampling $n_\ell=30$ points from each of the two subspaces
and adding $\mc N(0, (\sigma^2/m) \mI)$ noise. 
 Figure \ref{fig:phasediag} shows the CE as a function of $\sqrt{d/p}$ and $\sigma$ for dimensionality reduction via GRP. 
 
The clustering condition in Theorem \ref{thm:tscrpnoisy} guarantees that TSC succeeds as long as  $\sqrt{d/p} (c_1 + \sigma (c_2 + \sigma)) \leq c_3$, where $c_1,c_2,c_3$ are independent of $d,p,m$, and $\sigma^2$. 
In order to find out whether this sufficient condition predicts the fundamental clustering behavior qualitatively correctly, we test whether a phase transition, separating the region where clustering succeeds from that where it fails, indeed, occurs at 
\begin{align}
\sqrt{d/p} (c_1 + \sigma (c_2 + \sigma)) = c_3.
\label{eq:phastransline}
\end{align} 
To this end, we fit \eqref{eq:phastransline}---by choosing $c_1,c_2,c_3$---into the plots in  Figure \ref{fig:phasediag} and observe that the answer is in the affirmative. Moreover, our numerical results show that the phase transition behavior of SSC and SSC-OMP is essentially identical to that of TSC, which provides evidence for SSC and SSC-OMP behaving similarly to TSC in the noisy case.

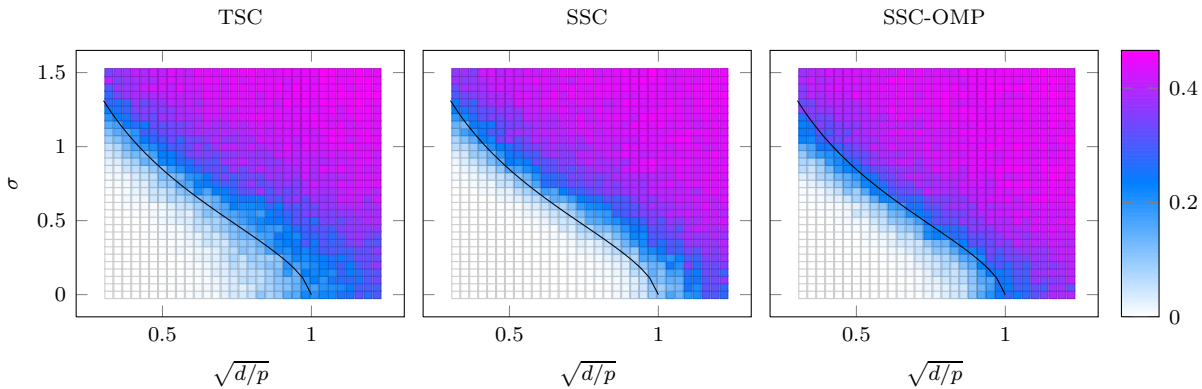
\begin{figure}
\centering
{
\pgfplotsset{colormap/cool} 
\begin{tikzpicture}[scale=1,font=\scriptsize]
\begin{groupplot}[%
      group style={%
        group size=3 by 1,
                horizontal sep=0.25cm,vertical sep=2.2cm,
        ylabels at=edge left,yticklabels at=edge left,
      },
      xlabel = $\sqrt{d/p}$,
      ylabel = $\sigma$,
      width = 0.36\textwidth
      ]

   \nextgroupplot[title=TSC,clip marker paths=true]
 \addplot[mark=square*,only marks, scatter, scatter src=explicit,
 mark size=1.5]
  file {./data/PhaseSynthGRP_TSC.dat};

\addplot[mark=none, black]
    gnuplot[id=log,domain=0.302:1]{ 
    (-0.1 +sqrt( 0.1^2 - 4*(0.8 - 0.8/x ) ))/2
    };
  \nextgroupplot[title=SSC,clip marker paths=true]
  \addplot[mark=square*,only marks, scatter, scatter src=explicit,mark size=1.5]
  file {./data/PhaseSynthGRP_SSC.dat};
    \addplot[mark=none, black]
    gnuplot[id=log,domain=0.302:1]{ 
    (-0.1 +sqrt( 0.1^2 - 4*(0.8 - 0.8/x ) ))/2
    };
  \nextgroupplot[title=SSC-OMP,colorbar,clip marker paths=true]
  \addplot[mark=square*,only marks, scatter, scatter src=explicit,mark size=1.5]
  file {./data/PhaseSynthGRP_SSC-OMP.dat};
  \addplot[mark=none, black]
    gnuplot[id=log,domain=0.302:1]{ 
    (-0.1 +sqrt( 0.1^2 - 4*(0.8 - 0.8/x ) ))/2
    };
\end{groupplot}
\end{tikzpicture}
}
\caption{\label{fig:phasediag}
CE (color coded) as a function of $\sqrt{d/p}$ and $\sigma$ for $L=2$ orthogonal subspaces of $\reals^{100}$.
The black lines correspond to the curve $\sqrt{d/p}(0.8 + \sigma(0.1+\sigma)) = 0.8$, and roughly separate the regimes where clustering succeeds from that where it fails. 
}
\end{figure}

\subsubsection{Dimensionality reduction when the subspaces span the ambient space}

As noted in Section \ref{sec:perfguar}, Theorems \ref{thm:tscrp}--\ref{thm:tscrpnoisy} 
indicate that dimensionality reduction down to the order of the subspace dimensions is possible even when the subspaces $\cS_\l$ span the ambient space $\reals^m$. To verify this observation empirically, we perform the following experiment. We draw a random Gaussian matrix $\mV \in \reals^{200\times 200}$. With probability one, the columns of $\mV$ span $\reals^{200}$. We then extract the $200\times 20$ matrices $\mV^{(\ell)}$ from $\mV$ according to $[\mV^{(1)} \; \ldots \; \mV^{(10)}] = \mV$, and let the subspace $\cS_\ell$ be given by the span of $\mV^{(\ell)}, \ell=1,...,10$.  
This guarantees that the union of the $\cS_\ell$ span $\reals^{200}$. 
Note, however, that the affinities between pairs of the resulting subspaces will be small with high probability. We again use the data model described in Section \ref{sec:perfguar} and sample $n_\l = 60$ points on $\cS_\l \cap \US{d_\l}$, for all $\l \in [L]$, to obtain a data set $\YS$ with a total of $N = 600$ points. 
We select the values for $q$, $\lambda$, and $s_{\max}$ that  yield the lowest CE for the majority of values for $p$. 

Figure \ref{fig:ambspan} shows the CE as a function of $p$ for TSC, SSC, and SSC-OMP. 
The CE starts to be non-zero for $p < 60$ for TSC and SSC-OMP, and for $p < 40$ for SSC. We therefore conclude that the dimensionality can, indeed, be reduced, quite significantly, even when the subspaces span the ambient space, as indicated by Theorems \ref{thm:tscrp}--\ref{thm:tscrpnoisy}. 

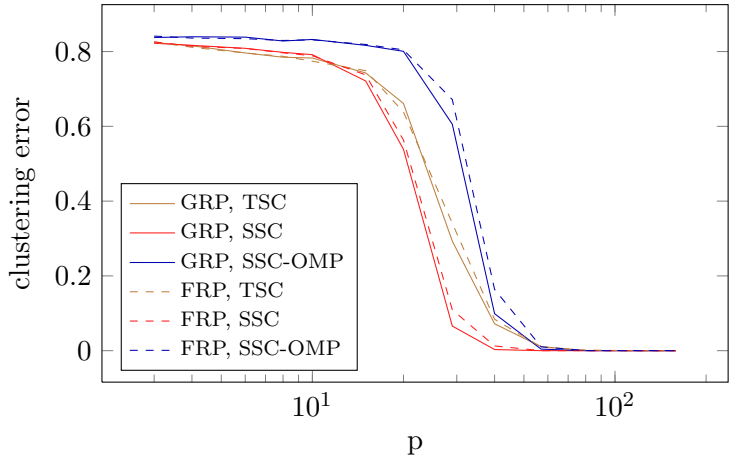
\begin{figure}[h!]
    
    \centering
\begin{tikzpicture}[scale=1] 
    \begin{semilogxaxis}[ name=plot1,
            xlabel=p,
            ylabel=clustering error,
    width =0.6\textwidth,
    height = 0.4\textwidth,
    legend entries={ {GRP, TSC}, {GRP, SSC}, {GRP, SSC-OMP}, { FRP, TSC}, { FRP, SSC}, { FRP, SSC-OMP}}, 
    legend style={
            cells={anchor=west},
                    legend pos= south west,
                    font=\scriptsize,}
    ]
    \addplot +[mark=none,solid,\plotgreen] table[x index=0,y index=2]{./data/CESynthFR.dat};
    \addplot +[mark=none,solid,\plotred] table[x index=0,y index=1]{./data/CESynthFR.dat};
    \addplot +[mark=none,solid,\plotblue] table[x index=0,y index=3]{./data/CESynthFR.dat};
    \addplot +[mark=none,dashed,\plotgreen] table[x index=0,y index=5]{./data/CESynthFR.dat};
    \addplot +[mark=none,dashed,\plotred] table[x index=0,y index=4]{./data/CESynthFR.dat};
    \addplot +[mark=none,dashed,\plotblue] table[x index=0,y index=6]{./data/CESynthFR.dat};
    \end{semilogxaxis}
    
\end{tikzpicture}
\caption[Subspaces spanning the ambient space]{\label{fig:ambspan}CE as a function of $p$ for $L=10$ subspaces that collectively span the ambient space $\reals^{200}$.}
\end{figure}

\subsection{Clustering faces}

We next evaluate the impact of dimensionality reduction 
in the problem of clustering face
 images taken from the Extended Yale B data set \cite{georghiades_illumination_2001,lee_acquiring_2005}, 
which contains $192 \times 168$ pixel ($m = 32256$) 
frontal face images of $38$ individuals, with $64$ images per individual, each acquired under different illumination conditions. 
The motivation for applying subspace clustering algorithms to this problem stems from the insight that the vectorized images of a given face taken under varying illumination conditions lie approximately in a 9-dimensional linear subspace \cite{basri_lambertian_2003}. Each 9-dimensional subspace $\cS_\l$ would then contain the images corresponding to a given person.

We generate $\YS$ by first selecting a subset of $L=2$ individuals uniformly at random from the set of all $38 \choose 2$ pairs 
and then collecting all images corresponding to the two selected individuals.
In Figure \ref{fig:fracesrprt}, we plot the corresponding CE and the running times as a function of $p$. 
Again, for each $p$, the CE and the running times are obtained by averaging over $500$ problem instances generated by randomly choosing $100$ instances of $\YS$ and $5$ realizations of the projection matrix per chosen data set $\YS$. In contrast to the preceding experiment, here, SSC-OMP consistently outperforms TSC and SSC. 
For all three algorithms the dimensionality of the data can be reduced by a factor of about $100$ without notably increasing the CE. 
Note, however, that in this experiment the dimensionality cannot be reduced as aggressively as in the preceding synthetic data experiment. Specifically, here the data points lie in $9$-dimensional subspaces and dimensionality reduction by a factor of $100$ corresponds to $p \approx 322$. 
One possible explanation for this observation is that the principal angles between the subspaces spanned by the face images of different subjects are typically small (see \cite[Sec. 7]{elhamifar_sparse_2013}), which means that the subspace affinities in this data set are large. 
The conclusions regarding running times and choice of the random projection matrix are analogous to those reported for synthetic data above.

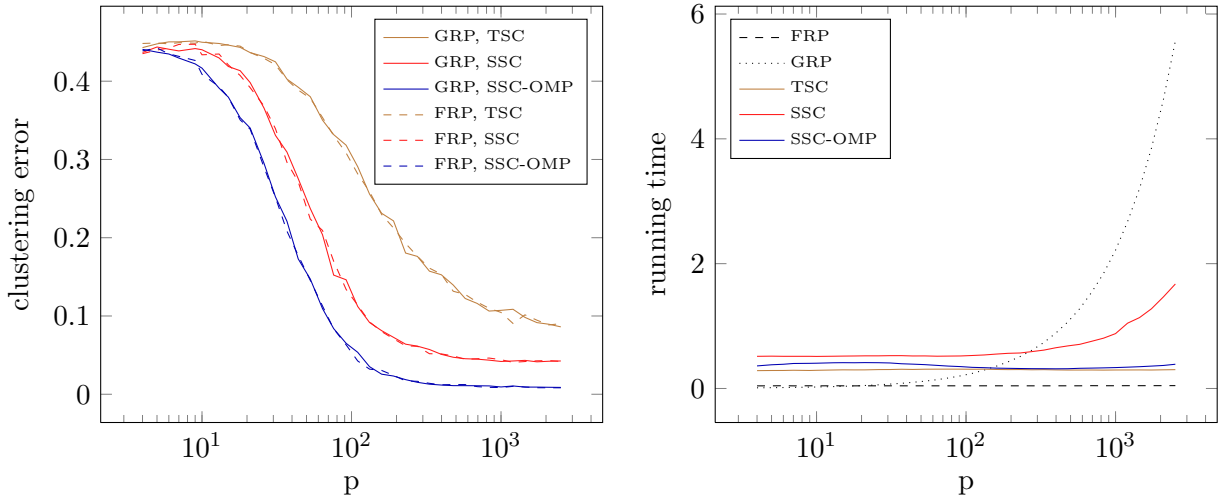
\begin{figure}[h!]
    
    \centering
\begin{tikzpicture}[scale=1]

   \begin{groupplot}[%
      group style={%
        group size=2 by 1,
           horizontal sep=1.5cm,vertical sep=0.93cm,
      },
      xlabel = p,
      ylabel = {clustering error},
      width = 0.5\textwidth
      ]

\nextgroupplot[xmode=log, ylabel= clustering error,
legend entries={ {GRP, TSC}, {GRP, SSC}, {GRP, SSC-OMP}, { FRP, TSC}, { FRP, SSC}, { FRP, SSC-OMP}}, 
    legend style={
            cells={anchor=west},
                    legend pos= north east,
                    font=\tiny,}
]%
    \addplot +[mark=none,solid,\plotgreen] table[x index=0,y index=2]{./data/CEFacesTSCb.dat};
    \addplot +[mark=none,solid,\plotred] table[x index=0,y index=1]{./data/CEFacesTSCb.dat};
    \addplot +[mark=none,solid,\plotblue] table[x index=0,y index=3]{./data/CEFacesTSCb.dat};
    \addplot +[mark=none,dashed,\plotgreen] table[x index=0,y index=5]{./data/CEFacesTSCb.dat};
    \addplot +[mark=none,dashed,\plotred] table[x index=0,y index=4]{./data/CEFacesTSCb.dat};
    \addplot +[mark=none,dashed,\plotblue] table[x index=0,y index=6]{./data/CEFacesTSCb.dat};

    \nextgroupplot[xmode=log, ylabel= running time,
    legend entries={ {FRP}, {GRP}, {TSC}, {SSC}, {SSC-OMP}}, 
    legend style={
            cells={anchor=west},
                    legend pos= north west,
                    font=\tiny,}
    ]%
    \addplot +[mark=none,dashed,black] table[x index=0,y index=2]{./data/TCompFacesTSCb.dat};
    \addplot +[mark=none,dotted,black] table[x index=0,y index=1]{./data/TCompFacesTSCb.dat};
    \addplot +[mark=none,solid,\plotgreen] table[x index=0,y index=4]{./data/TCompFacesTSCb.dat};
    \addplot +[mark=none,solid,\plotred] table[x index=0,y index=3]{./data/TCompFacesTSCb.dat};
    \addplot +[mark=none,solid,\plotblue] table[x index=0,y index=5]{./data/TCompFacesTSCb.dat};

\end{groupplot}    
   
\end{tikzpicture}
\caption[Running times and CE for clustering faces; singular values of data submatrices]{\label{fig:fracesrprt} Clustering error and running times (in seconds) for clustering $L=2$ faces from the Extended Yale B data set. 
}
\end{figure}

\subsection{Clustering handwritten digits}

In this experiment, we investigate the impact of dimensionality reduction in the context of clustering 
images of handwritten digits. We use the MNIST data set \cite{mnist_2013} containing $10,000$ images of (horizontally and vertically) aligned handwritten digits of size $28 \times 28$ pixels ($m= 784$). The motivation for employing subspace clustering in this context 
stems from the observation that vectorized images of different handwritten versions of the same digit tend to lie near a low-dimensional subspace \cite{hastie_metrics_1998}. 

We generate the data sets $\YS$ by selecting $250$ images (out of $1000$) uniformly at random from each of the sets corresponding to the digits $2$, $4$, and $8$.
There is no specific reason for our choice of the digits $2$, $4$, and 8; other combinations of three digits yield similar results.  However, some combinations of digits are more difficult to cluster than others, e.g.,  1 and 7 are ``closer'' (in terms of the affinities between the subspaces the corresponding images approximately lie in) than 1 and 8; clustering $1$ and $7$ therefore typically results in a larger error than clustering $1$ and $8$. 
The results depicted in Figure \ref{fig:digitsrprt} show that the dimensionality of the data set can be reduced from $m= 784$ to $p=200$, i.e., by a factor of $3.9$, without notably increasing the CE incurred by TSC and SSC. 
For sufficiently large $p$, TSC yields a slightly lower clustering error than SSC. SSC-OMP is outperformed considerably by the other two algorithms.
\begin{figure}[h]
    
    \centering
\begin{tikzpicture}[scale=1] 
    \begin{semilogxaxis}[name=plot1,anchor=west, title = clustering error, 
            xlabel=p,
    width =0.65\textwidth,
    height = 0.4\textwidth,
    legend entries={ {GRP, TSC}, {GRP, SSC}, {GRP, SSC-OMP}, { FRP, TSC}, { FRP, SSC},  { FRP, SSC-OMP}}, 
    legend style={
            cells={anchor=west},
                    legend pos= north east,
                    font=\scriptsize,}
    ]
    \addplot +[mark=none,solid,\plotgreen] table[x index=0,y index=2]{./data/CEDigitsTSCb.dat};
    \addplot +[mark=none,solid,\plotred] table[x index=0,y index=1]{./data/CEDigitsTSCb.dat};
    \addplot +[mark=none,solid,\plotblue] table[x index=0,y index=3]{./data/CEDigitsTSCb.dat};
    \addplot +[mark=none,dashed,\plotgreen] table[x index=0,y index=5]{./data/CEDigitsTSCb.dat};
    \addplot +[mark=none,dashed,\plotred] table[x index=0,y index=4]{./data/CEDigitsTSCb.dat};
    \addplot +[mark=none,dashed,\plotblue] table[x index=0,y index=6]{./data/CEDigitsTSCb.dat};
    \end{semilogxaxis}

\end{tikzpicture}
\caption[CE for clustering handwritten digits; singular values of data submatrices]{\label{fig:digitsrprt} Clustering error for handwritten digits 2,4, and 8 from the MNIST data set.
}
\end{figure}
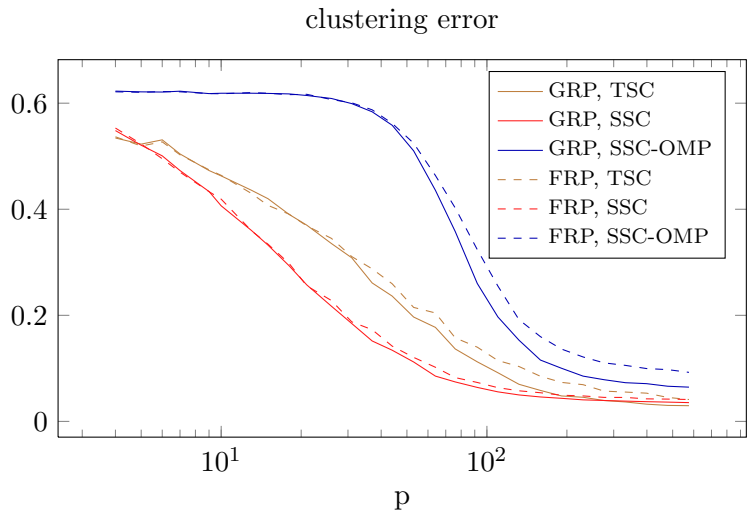

\subsection{Clustering gene expression data}

Finally, we consider clustering of 
gene expression level data--originating from different types of cancer cells--according to cancer type. 
This problem is of significant practical relevance as it helps, inter alia, to identify genes that are involved in the same cellular process \cite{jiang2004cluster}. The use of subspace clustering in this context was suggested in \cite{jiang2004cluster}.
We use the publicly available Novartis multi-tissue data set from the Broad Institute Cancer Program database \cite{broad_2013}. This data set contains the $1000$-dimensional gene expression level data of $n=103$ tissue samples taken from $L=4$ different cancer types.
In order to illustrate that the gene expression level vectors of a single cancer type, indeed, lie near a low-dimensional subspace, we plot, in Figure \ref{fig:cancerrprt}, the singular values of the data matrices corresponding to a single cancer type. We observe that the singular values decay rapidly and for every cancer type, more than $94\%$ of the energy of the corresponding data vectors is concentrated in a 6-dimensional subspace of the $1000$-dimensional ambient space.

We cluster all $n=103$ available samples. 
The CE obtained by averaging, for each $p$, over $200$ realizations of the random projection matrix 
is shown in Figure \ref{fig:cancerrprt}.
For $p \approx 100$, which corresponds to dimensionality reduction by a factor of $10$, the CEs of TSC and SSC are comparable to those obtained when operating on the original high-dimensional data set. SSC is seen to consistently (across $p$) perform best, followed by TSC and SSC-OMP. As in previous experiments the CEs observed for GRP and FRP, for each of the three algorithms, are virtually identical.

\begin{figure}[h!]
    
    \begin{tikzpicture}
   \begin{groupplot}[%
      group style={%
        group size=2 by 1,
           horizontal sep=1.5cm,vertical sep=0.93cm,
      },
      xlabel = p,
      ylabel = {clustering error},
      width = 0.5\textwidth
      ]

\nextgroupplot[xmode=log, ylabel= clustering error,
legend entries={ {GRP, TSC}, {GRP, SSC}, {GRP, SSC-OMP}, { FRP, TSC}, { FRP, SSC}, { FRP, SSC-OMP}}, 
    legend style={
            cells={anchor=west},
                    legend pos= north east,
                    font=\tiny,}
    ]
]%

    \addplot +[mark=none,solid,\plotgreen] table[x index=0,y index=2]{./data/CECancerTSCb.dat};
    \addplot +[mark=none,solid,\plotred] table[x index=0,y index=1]{./data/CECancerTSCb.dat};
    \addplot +[mark=none,solid,\plotblue] table[x index=0,y index=3]{./data/CECancerTSCb.dat};
    \addplot +[mark=none,dashed,\plotgreen] table[x index=0,y index=5]{./data/CECancerTSCb.dat};
    \addplot +[mark=none,dashed,\plotred] table[x index=0,y index=4]{./data/CECancerTSCb.dat};
    \addplot +[mark=none,dashed,\plotblue] table[x index=0,y index=6]{./data/CECancerTSCb.dat};

 \nextgroupplot[xmode=log, ylabel= magnitude of sigular value, label = index of singular value]%
	\foreach \idx in {1,...,4}
	{
		\addplot +[mark=none,solid,black,opacity=0.5] table[x index=0,y index=\idx]{./data/SingularValuesCancer.dat};
		}

 \end{groupplot}   
    
\end{tikzpicture}
\caption[CE for clustering gene expression level data; singular values of data submatrices]{\label{fig:cancerrprt}Clustering error for gene expression level data of $L=4$ cancer types (left). Singular values of data matrices corresponding to a single cancer type (right).}
\end{figure}
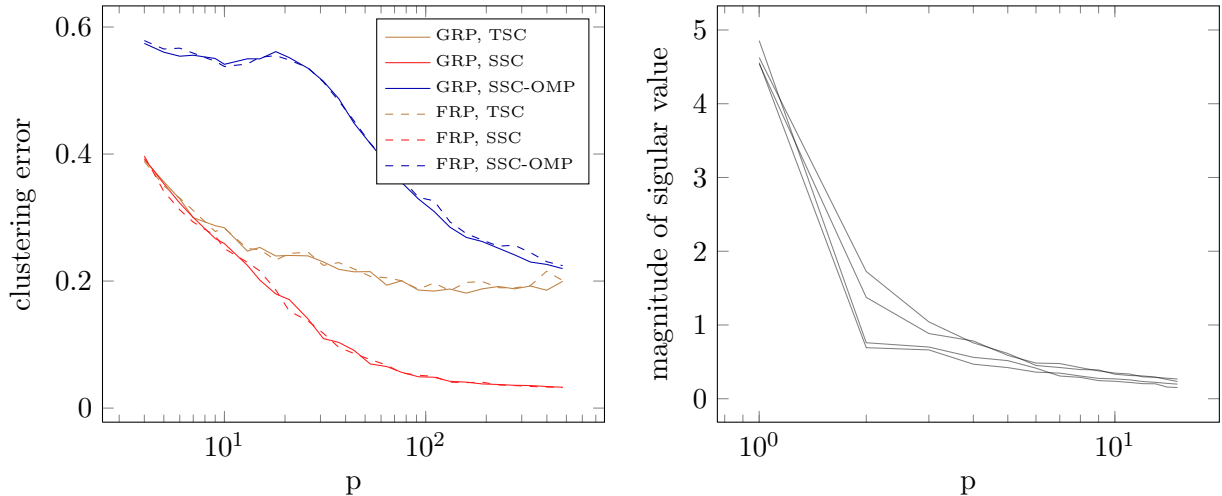

\renewcommand{\d}{d}

\section*{Acknowledgments}
The authors would like to thank Robert Calderbank for very helpful comments on an earlier version of this manuscript. 

\appendix



\section{\label{sec:proofRPTSC}Proof of Theorems \ref{thm:tscrp} and  \ref{thm:tscrpnoisy} }

The proof idea for Theorem \ref{thm:tscrp} is to turn the effect of the random projection into an additive perturbation and to show that this perturbation is small for all values of $p$ down to the order of $d_{\max}$. 
In the noisy case, addressed by Theorem \ref{thm:tscrpnoisy},  we have an additional perturbation due to noise. 
We detail the proof of the more general Theorem \ref{thm:tscrpnoisy} below, and explain in Appendix  \ref{sec:adthm:tscrp} the simple changes that yield Theorem \ref{thm:tscrp}. 
The proof of Theorem \ref{thm:tscrpnoisy} follows closely that of \cite[Thm.~3]{heckel_robust_2013}, which quantifies the performance of TSC under additive Gaussian noise alone. 
We therefore elaborate only on the steps that are new relative to \cite{heckel_robust_2013} and encourage the interested reader to consult \cite{heckel_robust_2013} for the arguments not repeated here. 

The graph $G$ obtained by applying TSC to the dimensionality-reduced noisy data set $\tilde \X$ has no false connections, i.e., each $\tilde \vx_i^{(\l)}$ is connected to points in $\tilde \X_\l$ only,  if for each $\tilde \vx_i^{(\l)} \in \tilde \X_\l$ the associated set $\S_i$ corresponds to points in $\tilde \X_\l$ only, for all $\l$. This is the case if 
\begin{align}
z_{(n_\l - \q)}^{(\l)} > \max_{k\neq \l, j} 
z_{j}^{(k)},
\label{eq:tscsdpfox2}
\end{align}
where $z_{j}^{(k)} \defeq \big| \big< \tilde \vx_j^{(k)} ,  \tilde \vx_i^{(\l)} \big> \big|$ and $z_{(1)}^{(\l)} \leq z_{(2)}^{(\l)} \leq \ldots \leq z_{(n_\l-1)}^{(\l)}$ are the order statistics of $\{z_{j}^{(\l)}\}_{j \in [n_\l] \setminus \{i\}}$ and $\max_{k\neq \l, j}$ denotes maximization over $k \in [L]$, $k \neq \l$, and over the indices $j$ of the corresponding points $\tilde \vx_j^{(k)} \in \tilde \X_k$. 
Note that, for simplicity of exposition, the notation $z_j^{(k)}$ does not reflect dependence on $\tilde \vx_i^{(\l)}$. 
The proof is established by upper-bounding the probability of \eqref{eq:tscsdpfox2} being violated for a given data point $\tilde \vx_i^{(\l)}$. A union bound over all $N$ points $\tilde \vx_i^{(\l)}, i\in [n_\l], \l\in [L]$, then yields the final result. 
We start by setting 
$
\bar z_j^{(k)} \defeq  \left| \innerprod{  \vy_j^{(k)} }{ \vy_i^{(\l)}  } \right|
$, 
where $\vy_j^{(k)} = \mU^{(k)} \va_j^{(k)}$ are the original data points in the (high-dimensional) space $\reals^m$, 
and noting that
$
z_j^{(k)} = \left| \innerprod{\tilde \vx_j^{(k)}}{ \tilde \vx_i^{(\l)} } \right| =   \left|  \innerprod{\vy_j^{(k)}}{ \vy_i^{(\l)} } +  e_j^{(k)}  \right|
$, 
where we defined the ``distortion''
\begin{align}
e_j^{(k)}
&\defeq  
\innerprod{ \Ph \tilde \vy_j^{(k)}}{ \Ph \tilde \vy_i^{(\l)} }   -  \innerprod{\vy_j^{(k)}}{ \vy_i^{(\l)} }  \nonumber \\
&= \innerprod{ (\transp{\Ph} \Ph - \mI )\vy_j^{(k)}}{\vy_i^{(\l)} } 
+ \innerprod{\mPh \vy_j^{(k)}}{\mPh \ve_i^{(\l)} }
+ \innerprod{\mPh \ve_j^{(k)}}{\mPh \vy_i^{(\l)} }
+ \innerprod{\mPh \ve_j^{(k)}}{\mPh \ve_i^{(\l)} }
\nonumber \\
&=  
\underbrace{
\innerprod{ \transp{\mU^{(\l)}} (\transp{\Ph} \Ph - \mI )\mU^{(k)} \va_j^{(k)}}{\va_i^{(\l)}}
}_{\bar e_j^{(k)}} \nonumber \\
&\hspace{0.5cm}+ 
\underbrace{
\innerprod{ \transp{\mPh}\mPh \vy_j^{(k)}}{ \ve_i^{(\l)} }
+ \innerprod{ \ve_j^{(k)}}{ \transp{\mPh}\mPh \vy_i^{(\l)} }
+ \innerprod{\transp{\mPh}\mPh \ve_j^{(k)}}{\ve_i^{(\l)} } 
}_{\tilde e_j^{(k)}}.
\label{eq:pertejk}
\end{align}
Here, the term $\bar e_j^{(k)}$ accounts for the perturbation caused by  random projection, and $\tilde e_j^{(k)}$ corresponds to the perturbation caused by noise. 
The probability of \eqref{eq:tscsdpfox2} being violated can now be upper-bounded according to 
\begin{align}
\PR{z_{(n_\l-\q)}^{(\l)} \leq \max_{k\neq \l, j} 
z_{j}^{(k)} } 
&\leq 
\PR{\bar z_{(n_\l-\q)}^{(\l)} \leq \frac{2}{3\sqrt{d_\l}}   }  \nonumber \\
&+\PR{ \max_{k\neq \l, j} 
\bar z_{j}^{(k)}  \geq \alpha  }  
+  \PR{   \max_{(j,k)\neq (i,\l)} \big|  e_j^{(k)} \big|   \,\geq \, \epsilon }\!,  
\label{eq:labprostobubb}
\end{align}
where we set 
\begin{align}
\alpha \defeq \frac{4\sqrt{6} \log N }{\sqrt{d_\l}} \max_{k \neq \l}
\frac{1}{\sqrt{d_k}}  \norm[F]{ \herm{\mU^{(k)}} \mU^{(\l)} },
\quad 
\epsilon \defeq 
\frac{\beta}{\sqrt{d_\l}} \delta 
+
\beta \frac{2\sigma(1+\sigma)}{\sqrt{m}} \delta'
\label{eq:defalep}
\end{align}
with $\beta \defeq \sqrt{6\log N}$, $\delta \defeq \allowbreak \frac{\sqrt{28 d_{\max} + 8 \log L + 8\log N }}{ \sqrt{ 3 \tilde c  p}}$, and $\delta' \defeq \sqrt{\frac{6m}{\bar c p}}$, 
and assumed 
\begin{equation}
\alpha + 2 \epsilon \leq \frac{2}{3\sqrt{d_\l}}. \label{eq:AssAlpha}
\end{equation} 
We refer the reader to \cite[Proof of Thm.~3, Eq.~(40)]{heckel_robust_2013} for an explanation of the steps leading to \eqref{eq:labprostobubb} (while \cite[Eq.~(40)]{heckel_robust_2013} is not completely equivalent to \eqref{eq:labprostobubb}, the steps leading to \eqref{eq:labprostobubb} are essentially identical). 
Resolving the assumption \eqref{eq:AssAlpha} leads to  
\[
\max_{k \neq \l}  
\frac{1}{\sqrt{d_k}}  \norm[F]{ \herm{\mU^{(k)}} \mU^{(\l)} } 
+ \frac{\delta}{2\sqrt{\log N}}  
+ \sigma (1+\sigma) \sqrt{\frac{6 \d_\l }{ \bar c p \log N }}
\leq \frac{2}{3\cdot 4 \sqrt{6}   \log N },
\]
which is implied by \eqref{eq:adikqopolnoisy}  (using that $\sqrt{28 d_{\max} + 8 \log L + 8 \log N }/\sqrt{\log N} \leq \sqrt{44 d_{\max}}$ because \linebreak $\log L / \log N \leq 1$, $d_{\max} \geq 1$, and $\log N > 1$ for $N \geq 3$). 
With $\epsilon$ as defined in \eqref{eq:defalep}, and the triangle inequality, 
it follows that 
$\max_{(j,k)\neq (i,\l)} \big|  e_j^{(k)} \big|   \,\geq \, \epsilon$ implies that 
either 
$\max_{(j,k)\neq (i,\l)} \big|  \bar e_j^{(k)} \big|   \,\geq \,   \frac{\beta}{\sqrt{d_\l}} \delta$
or 
$\max_{(j,k)\neq (i,\l)} \big|  \tilde e_j^{(k)} \big|   \,\geq \,  \beta \frac{2\sigma(1+\sigma)}{\sqrt{m}} \delta' $, or both. Therefore, by a union bound argument 
\begin{align}
\PR{   \max_{(j,k)\neq (i,\l)} \big|  e_j^{(k)} \big|   \,\geq \, \epsilon }
&\leq 
\PR{   \max_{(j,k)\neq (i,\l)} \big|  \bar e_j^{(k)} \big|   \,\geq \,   \frac{\beta}{\sqrt{d_\l}} \delta }
+
\PR{   \max_{(j,k)\neq (i,\l)} \big|  \tilde e_j^{(k)} \big|   \,\geq \,  \beta \frac{2\sigma(1+\sigma)}{\sqrt{m}} \delta' }. 
\label{eq:probsnoisuin} 
\end{align}
Here, the first and second term on the RHS of \eqref{eq:probsnoisuin} correspond to the perturbation caused by random projection and by noise, respectively. As established in Sections \ref{sec:bproj} and \ref{sec:protons} these terms can be upper-bounded by $\frac{4}{N^2}$ and  $2e^{-m} + \frac{7}{N^2}$, respectively, which yields
\begin{align}
\PR{   \max_{(j,k)\neq (i,\l)} \big|  e_j^{(k)} \big|   \,\geq \, \epsilon }
&\leq  \frac{4}{N^2} + 2e^{-m} + \frac{7}{N^2}. 
\label{eq:probbontotper}
\end{align}
The remaining terms on the RHS of \eqref{eq:labprostobubb} are upper-bounded as shown in Steps 3 and 2 in \cite[Proof of Thm.~3]{heckel_robust_2013}, respectively, using standard concentration of measure results,  according to
\[
\PR{\bar z_{(n_\l-\q)}^{(\l)} \leq \frac{2}{3\sqrt{d_\l}}   } 
\leq 
e^{-c (n_\l-1)}
\]
and 
\[
\PR{ \max_{k\neq \l, j} 
\bar z_{j}^{(k)}  \geq \alpha  }
\leq 
3N^{-2},
\]
where $c>1/20$ is a numerical constant, and we employed the assumption $n_\l \geq 6  q$, for all $\ell$. 

With \eqref{eq:probbontotper} we thus get 
that \eqref{eq:tscsdpfox2} is violated with probability at most $e^{-c(n_\l-1)} + 2e^{-m} + 14N^{-2}$. 
Taking the union bound over all points $\vx_i^{(\l)}, i \in [n_\l], \l \in [L]$, finishes the proof. 



\subsection{\label{sec:bproj} Perturbation caused by random projection}

We next show that the first term on the RHS of  \eqref{eq:probsnoisuin} is upper-bounded by  $4/N^2$. 
This term corresponds to the perturbation caused by random projection. 
For notational convenience, we set $\mB_{k,\ell} = \transp{\mU^{(\l)}} (\transp{\Ph} \Ph - \mI )\mU^{(k)}$ and note that  
\begin{align}
&\PR{   \max_{(j,k)\neq (i,\l)} \big|  \bar e_j^{(k)} \big|   \,\geq \, \frac{\beta}{\sqrt{d_\l}} \delta }
=
\PR{
\max_{(j,k)\neq (i,\l)}
\left| \innerprod{ \mB_{k,\ell} \va_j^{(k)}}{\va_i^{(\l)} } \right|
\geq
 \frac{ \beta }{\sqrt{d_\l}} \delta
} \nonumber \\
&=
\PR{
\bigcup_{(j,k)\neq (i,\l)}\left\{
\left| \innerprod{ \mB_{k,\ell} \va_j^{(k)}}{\va_i^{(\l)} } \right|
\geq
 \frac{ \beta }{\sqrt{d_\l}} \delta
\right\}
} \nonumber \\
&\leq 
\PR{
\bigcup_{(j,k)\neq (i,\l)}
\left\{ 
\left| \innerprod{ \mB_{k,\ell} \va_j^{(k)}}{\va_i^{(\l)} } \right|
\geq
 \norm[2]{ \mB_{k,\ell} \va_j^{(k)} }   \frac{ \beta }{\sqrt{d_\l}}
\right\}
\cup 
\left\{\norm[2]{ \mB_{k,\ell} \va_j^{(k)} } \geq \delta \right\}
} \nonumber \\
&\leq 
\PR{
\bigcup_{(j,k)\neq (i,\l)}
\left\{ 
\left| \innerprod{ \mB_{k,\ell} \va_j^{(k)}}{\va_i^{(\l)} } \right|
\geq
 \norm[2]{ \mB_{k,\ell} \va_j^{(k)} }   \frac{ \beta }{\sqrt{d_\l}}
\right\}
\cup 
\left\{\norm[2\to 2]{ \mB_{k,\ell} } \geq \delta \right\}
} \label{eq:usenoblale} \\
&\leq 
\PR{
\max_k
\norm[2\to 2]{\mB_{k,\ell}} \geq \delta
}
+
\sum_{(j,k)\neq (i,\l)}
\PR{ 
\left| \innerprod{ \mB_{k,\ell} \va_j^{(k)}}{\va_i^{(\l)} } \right|
\geq
 \norm[2]{ \mB_{k,\ell} \va_j^{(k)} }   \frac{ \beta }{\sqrt{d_\l}}
}
\label{eq:usenoblale2} \\
&\leq  2e^{- \tau/2} + N 2 e^{-\frac{6 \log N}{2}} = \frac{4}{N^2}, 
\label{eq:boundonerr}
\end{align}
where \eqref{eq:usenoblale} follows from $\norm[2]{ \mB_{k,\ell} \va_j^{(k)} }  \leq \norm[2 \to 2]{ \mB_{k,\ell}}$, \eqref{eq:usenoblale2} 
is by the union bound, and 
\eqref{eq:boundonerr} follows from \eqref{eq:prass2} in Appendix \ref{sec:proofRPBP} with $\tau = 4 \log N$ and Proposition \ref{thm:hoeffsphere} below with $\va = \va_i^{(\l)}$, $\vb = \mB_{k,\ell} \va_j^{(k)}$, $d=d_\ell$, and $\beta = \sqrt{6\log N}$. 

\begin{proposition}[{E.g.,~\cite[Ex.~5.25]{vershynin_introduction_2012}}]
Let $\va$ be uniformly distributed on $\US{d}$ and fix $\vb \in \reals^d$. Then, for $\beta\geq 0$, we have
\[
\PR{ \left|\innerprod{\va}{ \vb }\right|  > \frac{\beta}{\sqrt{d}} \norm[2]{ \vb} } 
\leq 2 e^{-\frac{ \beta^2}{2}}. 
\]
\label{thm:hoeffsphere}
\end{proposition}

%
%
%
%
%
%
%

\subsection{\label{sec:protons}Perturbation caused by noise}

In this section, we deal with the perturbation caused by noise. 
Specifically, we establish that the second term on the RHS of \eqref{eq:probsnoisuin} satisfies  
\begin{align}
\PR{
\max_{(j,k)\neq (i,\l)} \big|  \tilde e_j^{(k)} \big|   \,\geq \, \beta \frac{2\sigma(1+\sigma)}{\sqrt{m}} \delta'  
}
\leq 
2e^{-m} + \frac{7}{N^2}. 
\label{eq:noisypart}
\end{align}
For notational convenience, we set $\bar \vy_j^{(k)} = \transp{\mPh}\mPh \vy_j^{(k)}$ and  drop the indices $i$ and $\l$ to write $\vy = \vy_i^{(\l)}$, $\bar \vy = \bar \vy_i^{(\l)}$, $\ve = \ve_i^{(\l)}$. 
We first note that 
\begin{align}
&\left\{
\max_{(j,k)\neq (i,\l)} \big|  \tilde e_j^{(k)} \big|   \,\geq \, \beta \frac{2\sigma(1+\sigma)}{\sqrt{m}} \delta'  
\right\}
=
\bigcup_{(j,k)\neq (i,\l)}
\left\{ \big|  \tilde e_j^{(k)} \big|   \,\geq \, \beta \frac{2\sigma(1+\sigma)}{\sqrt{m}} \delta'  \right\}
 \nonumber \\
&\subseteq  
\bigcup_{(j,k)\neq (i,\l)}
\left\{ \abs{ \innerprod{\bar \vy_j^{(k)}}{ \ve } }  \geq \beta \frac{\sigma}{\sqrt{m}} \delta' \right\}
\cup 
\left\{ \abs{ \innerprod{ \ve_j^{(k)}}{ \bar \vy } } \geq \beta \frac{\sigma}{\sqrt{m}} \delta' \right\}  
\cup 
\left\{ \abs{ \innerprod{\transp{\mPh}\mPh \ve_j^{(k)}}{\ve } } \geq  \beta \frac{2\sigma^2}{\sqrt{m}} 
\delta' \right\} \label{eq:usetrioncag} \\
&\subseteq 
\left\{ \norm[2\to 2]{\transp{\mPh} \mPh} \geq \delta'  \right\} \cup 
\bigcup_{(j,k)\neq (i,\l)}
\left[
\left\{ 
\abs{\innerprod{ \bar \vy_j^{(k)} }{ \ve }}   \geq \beta \frac{\sigma}{\sqrt{m}} \norm[2]{\bar \vy_j^{(k)} }   
\right\} 
\cup 
\left\{ 
\abs{\innerprod{ \ve_j^{(k)} }{ \bar \vy }}   \geq \beta \frac{\sigma}{\sqrt{m}} \norm[2]{ \bar \vy }
\right\}
\right.   
\nonumber \\
&\hspace{0.5cm}\cup 
\left. \left\{
\abs{\innerprod{ \transp{\mPh}\mPh \ve_j^{(k)} }{ \ve }}   \geq \beta \frac{\sigma}{\sqrt{m}} \norm[2]{ \transp{\mPh}\mPh \ve_j^{(k)} }  
\right\} 
\cup 
\left\{
\norm[2]{ \ve_j^{(k)} } \geq 2 \sigma
\right\} \right].  \label{eq:eventinclus}
\end{align}
Here, \eqref{eq:usetrioncag} follows from the triangle inequality. 
To verify \eqref{eq:eventinclus}, consider the first event in \eqref{eq:usetrioncag} and note that 
\begin{align}
\left\{ \abs{ \innerprod{\bar \vy_j^{(k)}}{ \ve } }  \geq \beta \frac{\sigma}{\sqrt{m}} \delta' \right\}
\subseteq 
\left\{ \norm[2\to 2]{\transp{\mPh} \mPh} \geq \delta'  \right\}
\cup 
\left\{ 
\abs{\innerprod{ \bar \vy_j^{(k)} }{ \ve }} \geq \beta \frac{\sigma}{\sqrt{m}} \norm[2]{ \bar \vy_j^{(k)} }
\right\}. 
\label{eq:eventexte}
\end{align}
To see this, simply take the complement of \eqref{eq:eventexte}  according to 
\[
\left\{ \norm[2\to 2]{\transp{\mPh} \mPh} < \delta'  \right\}
\cap 
\left\{ 
\abs{\innerprod{ \bar \vy_j^{(k)} }{ \ve }} < \beta \frac{\sigma}{\sqrt{m}} \norm[2]{ \bar \vy_j^{(k)} }
\right\}
\subseteq 
\left\{ 
\abs{ \innerprod{\bar \vy_j^{(k)}}{ \ve } }  < \beta \frac{\sigma}{\sqrt{m}} \delta'
\right\}
\]
where we used
\[
\norm[2]{\bar \vy_j^{(k)}} 
=
\norm[2]{\transp{\mPh}\mPh \vy_j^{(k)}} 
\leq \norm[2\to 2]{\transp{\mPh}\mPh} \norm[2]{\vy_j^{(k)} } 
= \norm[2\to 2]{\transp{\mPh}\mPh}. 
\]
Treating the second and the third event in \eqref{eq:usetrioncag} similarly establishes  \eqref{eq:eventinclus}. 
A union bound argument now yields  
\begin{align}
\PR{
\max_{(j,k)\neq (i,\l)} \big|  \tilde e_j^{(k)} \big|   \,\geq \, 
\beta \frac{2\sigma(1+\sigma)}{\sqrt{m}} \delta'
}
\hspace{-1cm}&\hspace{1cm}\leq 
\PR{ \norm[2]{\transp{\mPh} \mPh} \geq \delta'}  \label{eq:ss0}\\
&+ \sum_{(j,k)\neq (i,\l)}
\PR{
\abs{\innerprod{ \bar \vy_j^{(k)} }{ \ve }}   \geq \beta \frac{\sigma}{\sqrt{m}} \norm[2]{\bar \vy_j^{(k)} }
}  \label{eq:ss1}\\
&+ \sum_{(j,k)\neq (i,\l)}
\PR{
\abs{\innerprod{ \ve_j^{(k)} }{ \bar \vy }}   \geq \beta \frac{\sigma}{\sqrt{m}} \norm[2]{ \bar \vy }
}   \label{eq:ss2}\\
&+ \sum_{(j,k)\neq (i,\l)}
\PR{
\abs{\innerprod{ \transp{\mPh}\mPh \ve_j^{(k)} }{ \ve }}   \geq \beta \frac{\sigma}{\sqrt{m}} \norm[2]{ \transp{\mPh}\mPh \ve_j^{(k)} }  
}  \label{eq:ss3}\\
&+ \sum_{(j,k)\neq (i,\l)}
\PR{
\norm[2]{ \ve_j^{(k)} } \geq 2 \sigma
} \label{eq:ss4} \\
&\leq     2 e^{-m} + 6 N e^{- \frac{\beta^2}{2}} + N e^{- \frac{\beta^2}{2}}. 
\label{eq:lastineq}
\end{align}
To get \eqref{eq:lastineq} we upper-bounded the terms on the RHSs of \eqref{eq:ss0}-\eqref{eq:ss4} as follows. 
For the RHS of \eqref{eq:ss0} we note that 
\[
\PR{ \norm[2]{\transp{\mPh} \mPh} \geq \delta'} \leq 2 e^{-m},
\]
 which is a consequence of Theorem \ref{thm:rauhutconc} stated in Appendix \ref{sec:proofRPBP} below. 
Specifically, with $1\leq  \sqrt{ \frac{6 m}{\bar c p} }$, which follows from $\bar c = \min(6, \tilde c) \leq 6$ and $p\leq m$, both by assumption, we have 
\begin{align}
\PR{\norm[2\to 2]{ \transp{\Ph} \Ph } \geq \sqrt{ \frac{24 m}{\bar c p} } }
&\leq 
\PR{\norm[2\to 2]{ \transp{\Ph} \Ph } \geq 1 + \sqrt{ \frac{6 m}{\tilde c p} } }  \nonumber \\
&\leq 
\PR{
\norm[2\to 2]{ \transp{\Ph} \Ph  - \mI } \geq \sqrt{ \frac{6 m}{\tilde c p} } 
}
\label{eq:useevinc} \\
&\leq 2 e^{- m } \label{eq:bythrauhuco}
%
\end{align}
where \eqref{eq:bythrauhuco} is by Theorem \ref{thm:rauhutconc} (with $t=\sqrt{2m}$). 
To establish \eqref{eq:useevinc}, first note that  $\norm[2\to 2]{ \transp{\Ph} \Ph  - \mI } \leq \delta'$ (with $\delta'=\sqrt{ \frac{6 m}{\bar c p}}$) implies $\sigmax{ \transp{\Ph} \Ph } \leq 1+\delta'$, which in turn is equivalent to   
$
\norm[2\to 2]{ \transp{\Ph} \Ph }  \leq 1 + \delta' 
$. 
We can therefore conclude that 
$
\norm[2\to 2]{ \transp{\Ph} \Ph }  \geq 1 + \delta' 
$
implies
$
\norm[2\to 2]{ \transp{\Ph} \Ph  - \mI } \geq \delta'
$. 

The terms inside the sums on the RHSs of \eqref{eq:ss1}, \eqref{eq:ss2}, and \eqref{eq:ss3}, were upper-bounded by applying Lemma \ref{lem:qfunction}, stated below. 
Specifically, we note that 
$\big< \bar \vy_j^{(k)} ,  \ve \big> \sim \mc N(0, \sigma^2  \big\| \bar \vy_j^{(k)} \big\|_2^2 )$,
 $\big<\ve_j^{(k)}  , \bar \vy \big> \sim \mc N(0, \sigma^2 \big\|  \bar \vy \big\|_2^2 )$, 
 and 
 $\big< \transp{\mPh}\mPh \ve_j^{(k)} ,  \ve \big> \sim \mc N(0, \sigma^2 \big\| \transp{\mPh}\mPh \ve_j^{(k)} \big\|_2^2 )$,
 where $\vy_j^{(k)}$, $\bar \vy$, and $\transp{\mPh}\mPh \ve_j^{(k)}$, respectively, 
 can be regarded as fixed, and we used $\beta =\sqrt{6\log N} \geq \frac{1}{\sqrt{2\pi}}$, as $N\geq 1$. 
\begin{lemma}[{\cite[Prop.~19.4.2]{lapidoth_foundation_2009}}]
Let $x\sim\mathcal N(0,1)$. For $\beta \geq \frac{1}{\sqrt{2\pi}}$, we have
\begin{equation}
\PR{x \geq \beta} \leq e^{- \frac{\beta^2}{2}}.
\label{eq:qfunctionb}
\end{equation}
\label{lem:qfunction}
\end{lemma}
Finally, to upper-bound the terms inside the sum in \eqref{eq:ss4}, we used \cite[Eq.~(51)]{heckel_robust_2013}
\begin{align}
\PR{  \norm[2]{ \ve_j^{(k)} }   \geq 2 \sigma  } 
\leq e^{-\frac{\beta^2}{2}}
\label{eq:adfj2}. 
\end{align}

\subsection{Proof of Theorem \ref{thm:tscrp}
\label{sec:adthm:tscrp}}

The proof of Theorem \ref{thm:tscrp} is obtained from the proof of Theorem \ref{thm:tscrpnoisy} by noting that in the noise-free case (i.e., $\sigma=0$), the perturbation caused by noise satisfies $\tilde e_j^{(k)} = 0$, rendering the second term on the RHS of \eqref{eq:probsnoisuin}  void. 
Finally, we remark that the assumption $6\log N \leq m$ is not needed in the noise-free case as it is involved only in establishing \eqref{eq:noisypart}, which is void here.

\section{Proof of Theorem \ref{th:RPNoiseless} \label{sec:proofRPBP}}


We first note that the data points in $\X_\ell$ can be written as $\vx_j^{(\l)} = \mV^{(\l)} \va^{(\l)}_j, j \in [n_\l]$, where the $\va^{(\l)}_j$ are i.i.d.~uniform on $\US{\d_\l}$, and $\mV^{(\l)} \defeq \mPh \mU^{(\l)}$ is a basis for the $d_\l$-dimensional subspace of $\reals^p$ containing the points in $\X_\ell$ ($\mV^{(\l)}$ has full column rank with high probability, which follows from \eqref{eq:prass1} as a consequence of the concentration inequality \eqref{eq:conceqcondonPh}). 
For the case where the $\mV^{(\ell)}$ are orthonormal bases a sufficient condition for successful clustering was derived by Soltanolkotabi and Cand\`es \cite[Thm.~2.8]{soltanolkotabi_geometric_2011}. 
However, owing to the projection $\mPh$, the $\mV^{(\l)} = \mPh \mU^{(\l)}$ will in general not be orthonormal. We will therefore need the following generalization of \cite[Thm.~2.8]{soltanolkotabi_geometric_2011} 
to arbitrary bases $\mV^{(\l)}$ for $d_\l$-dimensional subspaces of $\reals^p$. 

\begin{theorem}
\label{th:probrec}
Suppose that the elements of the sets $\X_\l$ in $\X = \X_1 \cup \ldots  \cup  \X_L$ 
are obtained by choosing $n_\l$ points at random according to $\vx_j^{(\l)} = \mV^{(\l)} \va^{(\l)}_j, j \in [n_\l]$, where the $\mV^{(\l)} \in \reals^{\p\times d_\l}$ have full rank, and the $\va^{(\l)}_j$ are i.i.d.~uniform on $\US{d_\l}$. Assume that $\rho_\l = (n_\l-1)/d_\l\geq \rho_0$, for all $\l$, where $\rho_0>1$ is a numerical constant, and let $\rho_{\min} = \min_\l \rho_\l$. If
\begin{equation}
\label{eq:ThmAffCondition2}
\max_{k, \l \colon k \neq \ell}  \frac{1}{\sqrt{ d_k}  }   \norm[F]{\pinv{\mV^{(\l)}}  \mV^{(k)}}
  \leq \frac{ \sqrt{\log \rho_{\min}} }{ 64 \log N},
\end{equation}
where $\pinv{\mV^{(\l)}} = \inv{(\transp{\mV^{(\l)}} \mV^{(\l)}  )}\transp{\mV^{(\l)}}$ is the pseudo-inverse of $\mV^{(\l)}$, 
then the graph $G$ with adjacency matrix obtained by applying SSC to $\X$ has no false connections with probability at least 
$1 -  N^{-1}  -  \sum_{\l=1}^L n_\l e^{-\sqrt{\rho_\l}  d_\l}$. 
\end{theorem}

\begin{proof}
See Appendix \ref{sec:pfnonorthbases}. 
\end{proof}
We now detail how Theorem \ref{th:RPNoiseless} follows from Theorem \ref{th:probrec}.  
Specifically, we will show that \eqref{eq:ThmAffConditionrp} implies \eqref{eq:ThmAffCondition2} with probability at least $1-4e^{-\tau/2}$, which, when combined with the probability bound in Theorem \ref{th:probrec} via a union bound yields the final probability estimate in Theorem \ref{th:RPNoiseless}, and thereby concludes the proof.

We start filling in the details by showing how \eqref{eq:ThmAffConditionrp} implies \eqref{eq:ThmAffCondition2}. The LHS of \eqref{eq:ThmAffCondition2} can be upper-bounded as follows  
\begin{align}
&\frac{1}{\sqrt{ d_k}} 
 \norm[F]{\pinv{\mV^{(\l)}}  \mV^{(k)}} 
 = \frac{1}{\sqrt{ d_k}  } \norm[F]{\big(\transp{\mV^{(\l)}} \mV^{(\l)}\big)^{-1} \transp{\mV^{(\l)}}  \mV^{(k)}   } \nonumber \\
 &\leq \norm[2\to 2]{\big(\transp{\mV^{(\l)}} \mV^{(\l)}\big)^{-1}}   \frac{1}{\sqrt{ d_k}  }   \norm[F]{\transp{\mV^{(\l)}}  \mV^{(k)}   } \label{eq:trivialnormineq} \\
 &\leq \frac{\norm[2\to 2]{\big(\transp{\mV^{(\l)}} \mV^{(\l)}\big)^{-1}}}{\sqrt{ d_k}  }\left( \norm[F]{\transp{\mU^{(\l)}}  \mU^{(k)}}  \!\!\! + \! \norm[F]{ \transp{\mU^{(\l)}} ( \transp{\mPh} \mPh -\mI)  \mU^{(k)}}   \right) \nonumber \\
  &\leq \norm[2\to 2]{\big(\transp{\mV^{(\l)}} \mV^{(\l)}\big)^{-1}} \left( \mathrm{aff}(\cS_k,\cS_\l)  + \norm[2\to 2]{ \transp{\mU^{(\l)}} ( \transp{\mPh} \mPh -\mI)  \mU^{(k)}}   \right) \label{eq:fnbopb} \\
&\leq \frac{1}{1-\delta} ( \mathrm{aff}(\cS_k,\cS_\l) + \delta ) \label{eq:withpradf} \\
&\leq \frac{65}{64} ( \mathrm{aff}(\cS_k,\cS_\l) + \delta ) \leq   \frac{ \sqrt{\log \rho_{\min}} }{64 \log N  }, \label{eq:byassth:RPNoiseless}
\end{align}
where 
\eqref{eq:trivialnormineq} follows from 
$\norm[F]{\mA \mB}^2 \leq \norm[2\to 2]{\mA}^2 \norm[F]{\mB}^2$, 
 \eqref{eq:fnbopb} is a consequence of $\norm[F]{\mB} \leq \sqrt{m \wedge n} \norm[2\to 2]{\mB}$, for $\mB \in \reals^{m\times n}$ \cite[Sec.~5.6, p.~365]{horn_matrix_2012},   
and \eqref{eq:withpradf} holds with 
\begin{align}
\delta \defeq \sqrt{ \frac{28d_{\max} + 8 \log L + 2\tau}{3\tilde c p} }, 
\label{eq:defdelta}
\end{align}
with probability at least $1-4e^{-\tau/2}$ (here, $\tau>0$ is the numerical constant in the statement of Theorem \ref{th:RPNoiseless}). 
Eq.~\eqref{eq:withpradf} holds with probability at least $1-4e^{-\tau/2}$ by 
\begin{align}
\PR{\max_{\l} \norm[2\to 2]{\big(\transp{\mV^{(\l)}} \mV^{(\l)}\big)^{-1}} \geq \frac{1}{1-\delta} } \leq 2e^{- \tau/2} 
\label{eq:prass1}
\end{align}
and 
\begin{align}
\PR{\max_{k, \ell} \norm[2\to 2]{ \transp{\mU^{(\l)}} ( \transp{\mPh} \mPh -\mI)  \mU^{(k)}}  \geq \delta } \leq 2e^{- \tau/2},
\label{eq:prass2}
\end{align}
both proven below. 
Finally, to get \eqref{eq:byassth:RPNoiseless} we invoked \eqref{eq:ThmAffConditionrp} twice, 
first we used $\mathrm{aff}(\cS_k,\cS_\l) \geq 0$ and $\frac{ \sqrt{\log \rho_{\min}} }{\log N} =\frac{ \sqrt{\log \rho_{\min}} }{\log\left(\sum_{\l=1}^L (\rho_\l d_\l+1) \right)}
 \leq 1$ in \eqref{eq:ThmAffConditionrp} to conclude that $\delta \leq 1/65$, i.e., $\frac{1}{1-\delta} \leq \frac{65}{64}$, and second, we applied \eqref{eq:ThmAffConditionrp} straight to upper-bound $\mathrm{aff}(\cS_k,\cS_\l)$.

It remains to prove \eqref{eq:prass1} and \eqref{eq:prass2}.  
For the special case of a Gaussian random matrix $\mPh$, the probability bounds \eqref{eq:prass1} and \eqref{eq:prass2} can be obtained using standard results on the extremal singular values of Gaussian random matrices. For general $\mPh$ satisfying the concentration inequality \eqref{eq:conceqcondonPh}, the proofs of \eqref{eq:prass1} and \eqref{eq:prass2} rely on Theorem \eqref{thm:rauhutconc} below.

\begin{theorem}[{\cite[Thm.~9.9, Rem.~9.10]{foucart_mathematical_2013}}]
Suppose that the random matrix $\mPh \in \reals^{p\times m}$ satisfies the concentration inequality \eqref{eq:conceqcondonPh}, i.e.,  
\[
\PR{ \left| \norm[2]{\mPh \vx}^2 - \norm[2]{\vx}^2 \right|  \geq t \norm[2]{\vx}^2 } \leq 2 e^{- \tilde c t^2 p},
\]
for all $t>0$ and for all $\vx \in \reals^{m}$, 
where $\tilde c$ is a constant. 
Then, for an orthonormal matrix $\mU \in \reals^{m\times d}$ and all $t>0$, we have 
\[
\PR{ \norm[2\to 2]{\transp{\mU} \transp{\mPh} \mPh \mU  - \mI } \geq  \sqrt{ \frac{14 d + 2t^2}{3\tilde c p} } } \leq 2 e^{-\frac{t^2}{2}}.
\]
Additionally, for all $t>0$, we have 
\[
\PR{ \norm[2\to 2]{ \transp{\mPh} \mPh - \mI } \geq  \sqrt{ \frac{14 m + 2t^2}{3\tilde c p} } } \leq 2 e^{-\frac{t^2}{2}}.
\]
\label{thm:rauhutconc}
\end{theorem}

\paragraph{Proof of \eqref{eq:prass1}:}
By a union bound argument, we get
\begin{align}
&\PR{\max_{\l} \norm[2\to 2]{\big(\transp{\mV^{(\l)}} \mV^{(\l)}\big)^{-1}} \geq \frac{1}{1-\delta} } 
\leq
\sum_{\l =1}^L \PR{ \norm[2\to 2]{\big(\transp{\mV^{(\l)}} \mV^{(\l)}\big)^{-1}} \geq \frac{1}{1-\delta} }.
\label{eq:unqul}
\end{align}
Note that $\big\|  \transp{\mV^{(\l)}} \mV^{(\l)}  - \mI \big\|_{2\to 2} \leq \delta$ implies that $\sigmin{\transp{\mV^{(\l)}} \mV^{(\l)}} \geq 1-\delta$, which in turn implies  
$
\big\| (\transp{\mV^{(\l)}} \mV^{(\l)})^{-1} \big\|_{2\to 2}  \leq \frac{1}{1 - \delta} 
$. 
We can therefore conclude that 
$
\big\| (\transp{\mV^{(\l)}} \mV^{(\l)})^{-1} \big\|_{2\to 2}  \geq \frac{1}{1 - \delta} 
$ 
implies 
$
\big\|  \transp{\mV^{(\l)}} \mV^{(\l)}  - \mI \big\|_{2\to 2} \geq \delta
$, 
which can be formalized according to
\[
\left\{
\big\| (\transp{\mV^{(\l)}} \mV^{(\l)})^{-1} \big\|_{2\to 2}  \geq \frac{1}{1 - \delta} 
\right\}
\subseteq
\left\{ 
\big\|  \transp{\mV^{(\l)}} \mV^{(\l)}  - \mI \big\|_{2\to 2} \geq \delta
\right\}.
\]
Moreover, 
we have  with $\delta$ as defined in \eqref{eq:defdelta} 
$\delta = \sqrt{ \frac{28d_{\max} + 2t^2}{3\tilde c p} } \geq \sqrt{ \frac{14d_\l + 2t^2}{3\tilde c p} }$  ($2d_{\max} \geq d_{\max}\geq d_\ell$), with $t^2 = 4 \log L + \tau$. Therefore, Theorem \ref{thm:rauhutconc} (with $\mU = \mU^{(\l)}$ and $t^2 = 4 \log L + \tau$) yields
\begin{align*}
\PR{\norm[2\to 2]{\big(\transp{\mV^{(\l)}} \mV^{(\l)}\big)^{-1}} \geq \frac{1}{1-\delta} }
&\leq 2e^{-2\log L - \tau/2} = 2L^{-2}e^{-\tau/2} \leq 2L^{-1}e^{-\tau/2},
\end{align*}
which when used on the RHS of \eqref{eq:unqul} establishes \eqref{eq:prass1}.  

\paragraph{Proof of  \eqref{eq:prass2}:}
Again, by a union bound argument, we get
\begin{align}
&\PR{\max_{k,\ell} \norm[2\to 2]{ \transp{\mU^{(\l)}} ( \transp{\mPh} \mPh -\mI)  \mU^{(k)}}  \geq \delta } 
\leq 
\sum_{k,\ell=1}^L \PR{ \norm[2\to 2]{ \transp{\mU^{(\l)}} ( \transp{\mPh} \mPh -\mI)  \mU^{(k)}} \geq \delta }.
\label{eq:unqul2}
\end{align}
We next upper-bound the probabilities on the RHS of \eqref{eq:unqul2}. 
To this end, let $\tilde \mU \in \reals^{m\times \tilde d}$  be an orthonormal basis for the $\tilde d$-dimensional span of $[\mU^{(\l)}\,\, \mU^{(k)} ]$ ($\max(d_\l, d_k) \leq \tilde d \leq d_\ell + d_k$). Since $\tilde 
\mU \transp{\tilde \mU}$ is the orthogonal projection onto $\text{span}([\mU^{(\l)}\,\, \mU^{(k)} ])$, we have $\tilde 
\mU \transp{\tilde \mU} \mU^{(\l)}=  \mU^{(\l)}$ and $\tilde 
\mU \transp{\tilde \mU} \mU^{(k)}=  \mU^{(k)}$. 
Therefore, we get 
\begin{align}
\norm[2\to 2]{ \transp{\mU^{(\l)}} ( \transp{\mPh} \mPh -\mI)  \mU^{(k)}} 
\!\! &= \norm[2\to 2]{ \transp{\mU^{(\l)}} \tilde \mU \transp{\tilde \mU}   ( \transp{\mPh} \mPh -\mI) \tilde \mU \transp{\tilde \mU} \mU^{(k)}} \nonumber \\
&\leq 
\norm[2\to 2]{ \transp{\mU^{(\l)}} \tilde \mU} \norm[2\to 2]{ \transp{\tilde \mU} \transp{\mPh} \mPh \tilde \mU -\mI} \norm[2\to 2]{ \transp{\tilde \mU} \mU^{(k)} }\nonumber \\
&= 
 \norm[2\to 2]{ \transp{\tilde \mU} \transp{\mPh} \mPh \tilde \mU -\mI}, \nonumber
\end{align}
where we used $\norm[2\to 2]{ \transp{\mU^{(\l)}} \tilde \mU}=1$, which holds since $\mU^{(\l)}$ is in the span of $\tilde \mU$ and both $\mU^{(\l)}$ and $\tilde \mU$ are orthonormal. 
This finally yields, with $\delta$ as defined in \eqref{eq:defdelta}, 
\begin{align}
\PR{ \norm[2\to 2]{ \transp{\mU^{(\l)}} ( \transp{\mPh} \mPh -\mI)  \mU^{(k)}} \geq \delta }  
&\leq 
\PR{ \norm[2\to 2]{ \transp{\tilde \mU} \transp{\mPh} \mPh \tilde \mU -\mI}  \geq \sqrt{ \frac{28d_{\max} + 8 \log L + 2\tau}{3\tilde c p} } } \nonumber \\ 
&\leq 
\PR{ \norm[2\to 2]{ \transp{\tilde \mU} \transp{\mPh} \mPh \tilde \mU -\mI}  \geq \sqrt{ \frac{14 \tilde d + 8 \log L + 2\tau}{3\tilde c p} } } \label{eq:tildineq} \\
&\leq 
2 e^{-\frac{4 \log L + \tau}{2}}
=2L^{-2}e^{- \tau/2},  
\label{eq:finalperuqpiub}
\end{align}
where \eqref{eq:tildineq} follows from $2d_{\max} \geq d_\l +\d_k \geq \tilde d$, and \eqref{eq:finalperuqpiub} is by application of Theorem \ref{thm:rauhutconc} with $\mU = \tilde \mU$ and  $t^2 = 4 \log L + \tau$. 
The proof is concluded by using \eqref{eq:finalperuqpiub} on the RHS of  \eqref{eq:unqul2}. 

\subsection{Proof of Theorem \ref{th:probrec}}
\label{sec:pfnonorthbases}
Theorem \ref{th:probrec} is a generalization of a result by Soltanolkotabi and Cand\`es \cite[Thm.~2.8]{soltanolkotabi_geometric_2011} from orthonormal bases $\mV^{(\l)}$ for $d_\l$-dimensional subspaces of $\reals^p$ to arbitrary bases $\mV^{(\l)}$ for $d_\l$-dimensional subspaces. 
The proof program essentially follows that of \cite[Thm.~2.8]{soltanolkotabi_geometric_2011}. However, some parts of the generalization are non-trivial. 
We only detail the arguments that are new relative to \cite{soltanolkotabi_geometric_2011}, and refer to \cite{soltanolkotabi_geometric_2011} otherwise.

Throughout the proof, we use the following notation: 
Let $\mX^{(\l)} \in \reals^{p\times n_\l}$ be the matrix whose columns are the points in $\X_\l$, and note that $\mX^{(\l)} = \mV^{(\l)} \mA^{(\l)}$, where $\mA^{(\l)} \in \reals^{d_\l\times n_\l}$ is the matrix with columns $\va_i^{(\l)}, i=1,\ldots,n_\l$.  
Set $\mX = [\mX^{(1)} \; \ldots\; \mX^{(L)} ] \in \reals^{p\times N}$, and let $\mX_{-i}$ be the matrix obtained by removing the $i$th column $\vx_i$ from $\mX$.  
 $\mc P(\mX)$ denotes the symmetrized convex hull of the columns of $\mX$ (i.e., the points in $\X$), that is, the convex hull of $\{\vx_1,-\vx_1,\ldots,\vx_N,-\vx_N\}$. 
For a convex body $\mathcal P$, its inradius $r(\mathcal P)$ is defined as the radius of the largest Euclidean ball that can be inscribed in $\mathcal P$, and its circumradius $R(\mathcal P)$ is defined as the radius of the smallest ball containing $\mathcal P$. Finally, the polar set of $\mc K \subset \reals^n$ is defined as
\[
\mc K^\circ = \{ \vy \in \reals^n\colon  \innerprod{\vx}{\vy}  \leq 1 \text{ for all } \vx \in \mc K  \}. 
\]

\subsubsection{A deterministic clustering condition}

We first establish a deterministic clustering condition. 
Specifically, in Theorem \ref{thm:suffcondSSCgen} below we present conditions
guaranteeing that for 
$\vx_i \in \X_\l$ every solution of the problem 
  \begin{align}
  \underset{\vz}{\text{minimize }} \norm[1]{\vz} \;\text{ subject to }  \mX_{-i} \vz = \vx_i 
  \label{eq:minsscw}
  \end{align}
has non-zero entries corresponding to columns of $\mX^{(\l)}$ only. 
The proof of Theorem \ref{th:probrec} is then obtained by 
proving that these conditions 
are satisfied with high probability for the statistical data model in Theorem \ref{th:probrec}. 
We start by introducing terminology needed in the following. 
Define the primal optimization problem 
\[
P(\vy, \mA) \colon \underset{\vz}{\text{minimize }} \norm[1]{\vz}  \text{ subject to } \mA \vz = \vy
\]
with the corresponding dual \cite[Sec.~5.1.16]{boyd_convex_2004}
\[
D(\vy,\mA) \colon \underset{\vnu}{\text{maximize}} \innerprod{\vy}{\vnu} \text{ subject to } \norm[\infty]{\transp{\mA} \vnu} \leq 1.
\]
The problem \eqref{eq:minsscw} is then simply $P(\vx_i, \mX_{-i})$. The sets of optimal solutions of $P$ and $D$ are denoted by $\text{optsol} P(\vy, \mA)$ and $\text{optsol} D(\vy, \mA)$, respectively. 
A dual point $\vla(\vy,\mA)$ is defined as a point in $\text{optsol} D(\vy, \mA)$ of minimal Euclidean norm.

We are now ready to state the following generalization of \cite[Thm.~2.5]{soltanolkotabi_geometric_2011} from orthonormal bases $\mV^{(\l)}$ for $d_\l$-dimensional subspaces of $\reals^p$ to arbitrary bases $\mV^{(\l)}$ for $d_\l$-dimensional subspaces. 

\begin{theorem}
Suppose that the elements of the sets $\X_\l$ in $\X = \X_1 \cup \ldots  \cup  \X_L$ 
are obtained by choosing $n_\l$ points according to $\vx_i^{(\l)} = \mV^{(\l)} \va_i^{(\l)}, i \in [n_\l]$, where the $\va_i^{(\l)}$ are deterministic coefficient vectors and the $\mV^{(\l)} \in \reals^{\p \times d_\l}$ are deterministic matrices of full column rank. Let $\mL \in \reals^{d_\l \times (n_\ell - 1)}$ be the matrix whose columns are the normalized dual points $\tilde \vla(\va_i^{(\l)}, \mA_{-i}^{(\l)}) = \vla(\va_i^{(\l)}, \mA_{-i}^{(\l)}) / \big\| \vla(\va_i^{(\l)}, \mA_{-i}^{(\l)}) \|_2$, where $\mA^{(\l)}_{-i}$ is the matrix with columns $\va_j^{(\ell)}, j  \in [n_\ell] \setminus \{i\}$. 
If 
\begin{equation}
\max_{k \neq \l, j}   \norm[\infty]{  \transp{\mL} \pinv{\mV^{(\l)}}  \mV^{(k)} \va_j^{(k)}  } <  r(\mc P(\mA^{(\l)}_{-i})),
\label{eq:generalcond}
\end{equation}
then the non-zero entries of all solutions of $P(\vx_i^{(\l)}, \mX_{-(i,\ell)})$ correspond to points in $\X_\l$ only $($the columns of $\mX_{-(i,\ell)}$ are the elements in $\X \setminus \{\vx_i^{(\l)} \}$$)$. 
\label{thm:suffcondSSCgen}
\end{theorem}

\begin{proof}
The proof relies on the following lemma.

\begin{lemma}[{\cite[Lem.~7.1]{soltanolkotabi_geometric_2011}, \cite{candes_robust_2006}}]
Let $T$ be a subset of the column indices of a given matrix $\mA$. 
All solutions $\vc^\star$ of $P(\vy, \mA)$ satisfy $\vc^\star_{\comp{T}} = \vect{0}$, if
 there exists a vector $\vc$ such that $\vy=\mA \vc$ 
with support $\S \subseteq T$, and a (dual certificate) vector $\vnu$ obeying 
\begin{align}
\transp{\mA}_\S \vnu = \sgn(\vc_\S) \label{eq:nsc1}\\
\norm[\infty]{ \transp{\mA}_{T \cap \comp{\S} } \vnu } \leq 1 \label{eq:nsc2} \\
\norm[\infty]{ \transp{\mA}_{\comp{T} } \vnu } < 1.\label{eq:nsc3}
\end{align}
\label{le:troppfuchsonc}
\end{lemma}
We apply Lemma \ref{le:troppfuchsonc} with $\mA = \mX_{-(i,\ell)}$, $\vy=\vx_i^{(\l)}$, and $T$ the index set corresponding to the columns of $\mX^{(\l)}_{-i}$,  
and show that there exists a vector $\vc$ supported on $\S \subseteq T$ that obeys $\vx_i^{(\l)} = \mX_{-(i,\ell)} \vc$, and a corresponding vector $\vnu$ that satisfies \eqref{eq:nsc1}--\eqref{eq:nsc3}. 
This then implies that the non-zero entries of all solutions of $P(\vx_i^{(\l)}, \mX_{-(i,\ell)})$ correspond to points in $\X_\l$ only, as desired.  

We proceed with the explicit construction of the vector $\vc$. 
Specifically, take $\vc$ to be a vector that is zero on $\comp{T}$, and whose restriction to the index set $T$ is given by $\vc_T \in \text{optsol} P(\vx_i^{(\l)}, \mX_{-i}^{(\l)})$. Let $\S$ be the support of $\vc_T$, and let $\vnu_i^{(\l)} = \pinv{(\transp{  \mV^{(\l)}  })} \vla_i^{(\l)}$,  
where $\vla_i^{(\l)}$ is taken to be a point of minimum $\ell_2$-norm\footnote{
For concreteness $\vla_i^{(\l)}$ is taken to be a point of minimum $\ell_2$-norm. 
Note, however, that for the proof to work we may let $\vla_i^{(\l)}$ be an arbitrary point in $\text{optsol} P(\va_i^{(\l)}, \mA_{-i}^{(\l)})$.  
} in $\text{optsol} D(\va_i^{(\l)}, \mA_{-i}^{(\l)})$. 
The next step is to show that $\vnu_i^{(\l)} \in \text{optsol} D(\vx_i^{(\l)}, \mX_{-i}^{(\l)})$, which will eventually allow us to establish that $\vnu_i^{(\l)}$ satisfies the conditions of Lemma \ref{le:troppfuchsonc}. 
To this end, we first note that  $\vx_i^{(\l)} = \mV^{(\l)} \va_i^{(\l)}$ yields 
\begin{align*}
&\text{optsol} D(\vx_i^{(\l)}, \mX_{-i}^{(\l)}) \nonumber \\
&=\left\{\arg \max_\vnu \innerprod{\va_i^{(\l)}}{ \transp{\mV^{(\l)}}  \vnu} \text{ subject to } \norm[\infty]{\transp{(\mA_{-i}^{(\l)})} \transp{\mV^{(\l)}} \vnu} \leq 1 \right\}  \\
&=  \left\{ \vnu \colon    \vla = \transp{\mV^{(\l)}} \vnu, \, \vla \in \text{optsol} D(\va_i^{(\l)}, \mA_{-i}^{(\l)})  \right\} \\
&\supseteq 
\pinv{(\transp{ \mV^{(\l)} } )}  \text{optsol} D(\va_i^{(\l)}, \mA_{-i}^{(\l)}),
\end{align*}
where the inclusion holds as $\pinv{(\transp{\mV^{(\l)}  })} \vla$ is the minimum norm solution to the linear system of equations $\vla = \transp{\mV^{(\l)}  } \vnu$, but in general not the only solution.

Since $P(\vy, \mA)$ is a linear program, strong duality \cite[Sec.~5.2.3]{boyd_convex_2004} holds (provided that $P(\vy, \mA)$ is feasible) and therefore the optimal objective values of $P(\vx_i^{(\l)}, \mX_{-i}^{(\l)})$ and $D(\vx_i^{(\l)}, \mX_{-i}^{(\l)})$ coincide. It therefore follows that   
\begin{align}
\norm[1]{\vc_T} = \innerprod{\vx_i^{(\l)}}{ \vnu_i^{(\l)} }.
\label{eq:optequal}
\end{align}

Since $\vc_T \in \text{optsol} P(\vx^{(\l)}_i, \mX^{(\l)}_{-i})$ and $\vc_T$ is supported on $S$, 
 both by assumption, we have $\vx_i^{(\l)} = \mX^{(\l)}_{-i} \vc_T = (\mX^{(\l)}_{-i})_S \vc_S$, and therefore  \eqref{eq:optequal} becomes
\begin{align}
\innerprod{\vc_S}{\sgn(\vc_S)} = \innerprod{(\mX^{(\l)}_{-i})_S \vc_S }{ \vnu_i^{(\l)} } =  \innerprod{ \vc_S}{ \transp{(\mX^{(\l)}_{-i})}_S   \vnu_i^{(\l)} }. 
\label{eq:usingdualpointsat}
\end{align}
On the other hand, as
 $\vnu_i^{(\l)} \in \text{optsol} D(\vx_i^{(\l)}, \mX_{-i}^{(\l)})$, 
 we have $\norm[\infty]{ \transp{(\mX^{(\l)}_{-i})}  \vnu_i^{(\l)} } \leq 1$, which is equivalent to the following conditions (recall that the set $T$ corresponds to the column indices of $\mX^{(\l)}_{-i}$): 
\begin{equation}
\norm[\infty]{ \transp{((\mX^{(\l)}_{-i})_S)}  \vnu_i^{(\l)} } \leq 1
\label{eq:condXsnu}
\end{equation}
\begin{equation}
\norm[\infty]{ \transp{((\mX^{(\l)}_{-i})_{T \cap \comp{S}})}  \vnu_i^{(\l)} } \leq 1.
\label{eq:equivcond5}
\end{equation}
As by \eqref{eq:condXsnu}, the entries of 
$\transp{(\mX^{(\l)}_{-i})}_S  \vnu_i^{(\l)}$
are bounded in magnitude by $1$ and the unique maximizer of $\max_{\va\colon \norm[\infty]{\va} \leq 1}\innerprod{\vc_S}{\va}$ is $\sgn(\vc_S)$, it follows from \eqref{eq:usingdualpointsat}  
that
\[
\transp{(\mX^{(\l)}_{-i})}_S  \vnu_i^{(\l)} = \sgn(\vc_S),
\]
which establishes \eqref{eq:nsc1}.

Thanks to \eqref{eq:equivcond5}, \eqref{eq:nsc2} is satisfied as well. 
It remains to verify \eqref{eq:nsc3}, which here reads  
\begin{align}
\left|\innerprod{\vx_j^{(k)}}{ \vnu_i^{(\l)} }\right| < 1, \; \text{ for all } k\neq \l, \text{ for all } j\in [n_k]. 
\label{eq:ftcon3}
\end{align}
With $\vnu_i^{(\l)} = \pinv{(\transp{\mV^{(\l)}})} \vla_i^{(\l)}$, by definition,  \eqref{eq:ftcon3} becomes 
\begin{align}
\left|\innerprod{ \vx_j^{(k)} }{\pinv{(\transp{\mV^{(\l)}})}  \frac{ \vla_i^{(\l)} }{\norm[2]{\vla_i^{(\l)}}} }\right| <\frac{1}{\norm[2]{\vla_i^{(\l)}}},  \text{ for all } k\neq \l, \text{ for all } j \in [n_k].
\label{eq:adfooo94}
\end{align}
Since $(\pinv{(\transp{\mV^{(\l)}} ) })^T = \pinv{\mV^{(\l)}}$, and $\vx_j^{(k)} = \mV^{(k)} \va_j^{(k)}$,  \eqref{eq:adfooo94} is equivalent to
\begin{align}
\left| \frac{ \transp{\vla_i^{(\l)}} }{\norm[2]{\vla_i^{(\l)}}}  
 \pinv{\mV^{(\l)}}  \mV^{(k)} \va_j^{(k)} 
\right|
<\frac{1}{\norm[2]{\vla_i^{(\l)}}},  \text{ for all } k\neq \l, \text{ for all } j \in [n_k].
\label{eq:contdsfadsf}
\end{align}
It now follows from $\vla^{(\l)}_i \in \text{optsol} D(\va_i^{(\l)}, \mA_{-i}^{(\l)})$ which holds by assumption, 
that 
\[
\big\| (\mA^{(\l)}_{-i})^T \vla^{(\l)}_i  \big\|_\infty \leq 1.
\] 
This, in turn, implies that $\vla^{(\l)}_i \in \mathcal P^\circ(\mA^{(\l)}_{-i})$ where 
\[
\mc P^\circ(\mA^{(\l)}_{-i})  = \left\{\vz \colon   \norm[\infty]{(\mA^{(\l)}_{-i})^T \vz  } \leq 1  \right\}
\]
is the polar set of $\mc P(\mA^{(\l)}_{-i})$ (recall that $\mc P(\mA^{(\l)}_{-i})$  is the symmetrized convex hull of the columns in $\mA^{(\l)}_{-i}$). 
Since the inradius and the circumradius of a 
symmetric\footnote{A convex body $\mathcal P$ is called symmetric if $\vx \in \mathcal P$ if and only if $-\vx \in \mathcal P$.} convex body are related according to \cite[Thm.~1.2]{gritzmann_inner_1992} 
\[
r(\mathcal P) R(\mathcal P^\circ) = 1,
\]
we get from $\vla^{(\l)}_i \in \mathcal P^\circ(\mA^{(\l)}_{-i})$ that 
\begin{equation}
\norm[2]{\vla^{(\l)}_i} \leq R( \mc P^\circ(\mA^{(\l)}_{-i}) ) = \frac{1}{r(\mc P(\mA^{(\l)}_{-i}))}.
\label{eq:upblamil}
\end{equation}
By \eqref{eq:upblamil}, it follows that \eqref{eq:contdsfadsf} holds if
\[
\left| \frac{ \transp{\vla_i^{(\l)}} }{\norm[2]{\vla_i^{(\l)}}}  
 \pinv{\mV^{(\l)}}  \mV^{(k)} \va_j^{(k)} 
\right|
< r(\mc P(\mA^{(\l)}_{-i})),  \text{ for all } k\neq \l,\text{ for all } j \in [n_k],
\]
which is implied by \eqref{eq:generalcond}. This proves that \eqref{eq:nsc3} is satisfied as well, thereby concluding the proof. 
\end{proof}

\subsubsection{Evaluating the deterministic clustering condition for the statistical data model}

Theorem \ref{th:probrec} now follows from Theorem \ref{thm:suffcondSSCgen} by establishing that, for our statistical data model, the deterministic clustering condition \eqref{eq:generalcond} holds for all pairs $(\ell,i)$ with  $\l \in [L], i \in [n_\l]$, with high probability. 
Specifically, by a union bound argument, we get
\begin{align}
&\PR{ \text{\eqref{eq:generalcond} is violated for at least one pair } (\ell,i)} \nonumber \\
&\leq \sum_{(\ell,i)} \PR{\max_{k \neq \l, j}   \norm[\infty]{  \transp{\mL} \pinv{\mV^{(\l)}}  \mV^{(k)} \va_j^{(k)}  } \geq  r(\mc P(\mA^{(\l)}_{-i}))} \nonumber \\
&\leq 
\sum_{(\ell,i)}
\left(
\PR{
\max_{k \neq \l, j}  
\norm[\infty]{  \transp{\mL} \pinv{\mV^{(\l)}}  \mV^{(k)} \va_j^{(k)}  } 
\geq 
\frac{16\log N}{\sqrt{d_\l d_k}  }   \norm[F]{\pinv{\mV^{(\l)}}  \mV^{(k)}} }
+
\PR{ \frac{ \sqrt{\log \rho_\l}  }{4\sqrt{d_\l }} \geq r(\mc P(\mA^{(\l)}_{-i}))}
\right)
 \label{eq:usiolkipqnc} \\
&\leq  \sum_{\l=1}^L n_\l e^{-\sqrt{\rho_\l}  d_\l} + N^{-1}.  \label{eq:applyfinboundsaetci}
\end{align}
In \eqref{eq:usiolkipqnc} we used that for random variables $X$ and $Y$, possibly dependent, and constants $\phi$ and $\varphi$ satisfying $\phi \leq \varphi$, we have
 \begin{align}
\PR{X \geq Y} 
&\leq \PR{ \{X \geq \phi \} \cup \{\varphi \geq Y\} } \nonumber \\
&\leq \PR{X\geq \phi} + \PR{\varphi \geq Y}.
\label{eq:splitprob}
\end{align}
Specifically, we applied \eqref{eq:splitprob} with $\phi=\frac{16\log N}{\sqrt{d_\l d_k}  }   \norm[F]{\pinv{\mV^{(\l)}}  \mV^{(k)}}$ and $\varphi = \frac{ \sqrt{\log \rho_\l} }{4\sqrt{d_\l} }$, which leads to the assumption 
\[
\frac{16\log N}{\sqrt{d_\l d_k}  }   \norm[F]{\pinv{\mV^{(\l)}}  \mV^{(k)}}
  \leq \frac{ \sqrt{\log \rho_\l} }{4\sqrt{d_\l} }, \quad \text{for all } k, \l \colon k\neq \l, 
\]
implied by \eqref{eq:ThmAffCondition2}. To get  \eqref{eq:applyfinboundsaetci} we used that, for all $i$, 
\begin{align}
\label{eq:inradbound}
\PR{ \frac{ \sqrt{\log \rho_\l}  }{4\sqrt{d_\l }} \geq r(\mc P(\mA^{(\l)}_{-i}))  } \leq  e^{-\sqrt{\rho_\l}  d_\l}
\end{align}
and 
\begin{align}
\PR{
\max_{k \neq \l, j}  
\norm[\infty]{  \transp{\mL} \pinv{\mV^{(\l)}}  \mV^{(k)} \va_j^{(k)}  }  
\geq 
\frac{16\log N}{\sqrt{d_\l d_k}  }   \norm[F]{\pinv{\mV^{(\l)}}  \mV^{(k)}} } \leq N^{-2},
\label{eq:step1ubavva}
\end{align}
both of which are established next. 

The upper bound \eqref{eq:inradbound} is an application of \cite[Lem.~7.4]{soltanolkotabi_geometric_2011}, \cite{alonso_isotropy_2008}, and makes use of the assumption $(n_\l-1)/ d_\l = \rho_\l \geq \rho_0 > 1$. 
Finally, \eqref{eq:step1ubavva} follows from a union bound argument and 
\begin{align}
&\PR{\norm[\infty]{  \transp{\mL} \pinv{\mV^{(\l)}}  \mV^{(k)} \va_j^{(k)}  }  \geq  
\frac{16\log N}{\sqrt{d_\l d_k}  }   \norm[F]{\pinv{\mV^{(\l)}}  \mV^{(k)}}  }   \nonumber\\
&\hspace{4.9cm}\leq (n_\l+1) e^{- 4 \log N} \leq N^{-3},
\label{eq:upbonta}
\end{align}
which is a consequence of Lemma \ref{lem:maconclem} below together with the fact that the normalized dual point 
$\tilde \vla_i^{(\l)} = \vla_i^{(\l)} / \big\| \vla_i^{(\l)} \|_2$ is distributed uniformly on the unit sphere, as shown in \cite[Sec.~7.2.2 Proof of Step 2]{soltanolkotabi_geometric_2011}.

\begin{lemma}[{Extracted from the proof of Lemma 7.5 in \cite{soltanolkotabi_geometric_2011}}]
Let the columns of $\mL \in \reals^{d_1 \times n_1}$ be i.i.d.~uniform on $\US{d_1}$,  let $\va$ be uniform on $\US{d_2}$, and let $\mB \in \reals^{d_1 \times d_2}$. Then, for $c\geq12$, we have
\[
\PR{ \norm[\infty]{ \transp{\mL} \mB \va} \geq \frac{c}{\sqrt{d_1 d_2}} \norm[F]{\mB}   } \leq (n_1+1) e^{- \frac{c}{4}}.
\]
\label{lem:maconclem}
\end{lemma}

\newcommand{\x}[2]{\vx_{#1}^{(#2)}}
\newcommand{\xT}[2]{\vx_{#1}^{(#2)^T}}
\newcommand{\res}{\vr}
\newcommand{\rl}{\vr^{(\l)}}
\newcommand{\resp}{{\tilde \vr^{(\l)}}_{\pit}}

\newcommand{\U}[1]{\mU^{(#1)}}
\newcommand{\UT}[1]{\mU^{(#1)^T}}
\newcommand{\A}[1]{\mA^{(#1)}}
\newcommand{\AT}[1]{\mA^{(#1)^T}}
\renewcommand{\a}[2]{\va_{#1}^{(#2)}}
\newcommand{\aT}[2]{\va_{#1}^{(#2)^T}}
\newcommand{\eqdef}{=:}

\section{Proof of Theorem \ref{th:OMPRPnoiseless}} \label{sec:proofOMPSSC}

The graph $G$ obtained by SSC-OMP has no false connections if for each $\vx_i^{(\l)} \in \X_\l$ the OMP algorithm as detailed in Section \ref{sec:AlgIntro} 
selects points from $\X_\l$ only, for all $\l \in [L]$. 
This is the case if OMP selects points from $\X_\l$ in all iterations $s \in [\pit_\mathrm{end}]$ (we explain below that OMP terminates after $\pit_\mathrm{end} = \pit_{\max} \wedge d_\l$ iterations with high probability for our statistical data model). 
The OMP selection rule \eqref{eq:OMPSelRule} implies that OMP selects a point from $\X_\l$ in the $(s+1)$th iteration if 
\begin{equation}
\label{eq:CorrSelCond}
\max_{k \neq \l, j} \abs{ \innerprod{\x{j}{k}}{ \res_\pit}} < 
\max_{j \in [n_\l] \colon j \neq i} \abs{\innerprod{\x{j}{\l}}{ \res_\pit}}.
\end{equation}
Hence, the graph $G$ obtained by SSC-OMP has no false connections if the deterministic clustering condition \eqref{eq:CorrSelCond} holds for all $\pit_\mathrm{end}$ OMP iterations, for all $\vx_i^{(\l)} \in \X_\l, \l \in [L]$. 
We will next establish that \eqref{eq:CorrSelCond} is satisfied for our statistical data model with probability obeying the bound in Theorem \ref{th:OMPRPnoiseless}.

As a vehicle for our analysis, we introduce the \emph{reduced OMP algorithm} which, to compute sparse representations of the $\x{i}{\l}$, has access to the corresponding \emph{reduced data sets $\X_\l \! \setminus \! \{ \vx_i^{(\l)} \}$} only, instead of the full data sets $\X \! \setminus \! \{ \vx_i^{(\l)} \}$. 
If, for a given data set $\X$, the residuals computed by reduced OMP, henceforth denoted by $\rl_\pit$, satisfy \eqref{eq:CorrSelCond} for \emph{all iterations}, then the reduced OMP algorithm and the original OMP algorithm 
(processing the same data set $\X$) select exactly the same data points in the same order and we have $\vr_\pit = \rl_\pit$ for all $\pit \in [\pit_{\max} \wedge d_\l]$ by virtue of \eqref{eq:OMPSelRuleB}. 
We emphasize that for expositional convenience the notations $\rl_\pit$ and $\resp$ do not reflect dependence on $i$. The motivation for working with the reduced OMP algorithm is that $\rl_\pit$ being a function of the data points in $\X_\l$ only, conditionally on $\Ph$, is statistically independent of the data points in $\X \! \setminus \! \X_\l$. This will allow us to establish tail bounds for $\abs{ \innerprod{\smash{\x{j}{k}}}{ \smash{\rl_\pit}}}$, $k \neq \l$, $j \in [n_k]$, using standard concentration inequalities. 
We proceed to show that under the assumptions of Theorem \ref{th:OMPRPnoiseless} the reduced OMP residuals $\rl_\pit$ indeed satisfy \eqref{eq:CorrSelCond} for all $\l \in [L]$, $i \in [n_\l]$, and $\pit \in [\pit_{\max} \wedge d_\l]$ with probability meeting the lower bound in Theorem \ref{th:OMPRPnoiseless}.

Consider the reduced OMP algorithm for the data point $\vx_i^{(\l)}$ with fixed $\l \in [L]$ and fixed $i \in [n_\l]$. We start by noting that the reduced OMP index set $\Lambda_\pit$ is a function of the data points in $\X_\l$ only. After iteration $\pit$, with $\x{i}{\l} = \Ph \U{\l} \a{i}{\l}$ and $\mX^{(\l)}_{\Lambda_s} = \Ph \U{\l} \A{\l}_{\Lambda_\pit}$ inserted into \eqref{eq:OMPSelRuleB}, we get $\rl_\pit = \Ph\U{\l} \resp$, where
\[
\resp \defeq (\mI - \A{\l}_{\Lambda_\pit} \pinv{ (\Ph \U{\l} \A{\l}_{\Lambda_\pit})} \Ph \U{\l}) \a{i}{\l}.
\] 

We next establish a lower bound on the RHS of \eqref{eq:CorrSelCond} and an upper bound on the LHS of  \eqref{eq:CorrSelCond}. 
To isolate the impact of the different random quantities in the statistical data model, we will introduce events, upon the intersection of which \eqref{eq:CorrSelCond} is implied by \eqref{eq:ThmAffConditionOMPSSCRP} via these bounds. A union bound on the probability of the intersection of these events then yields the final result.

We start by lower-bounding the RHS of \eqref{eq:CorrSelCond} according to 
\begin{align}
 \max_{j \in [n_\l] \colon j \neq i} \abs{ \innerprod{\x{j}{\l}}{ \rl_\pit }} 
 &\geq
 \frac{1}{4} \sqrt{\frac{\log \rho_{\l}}{d_{\l}}} \sigmin{\UT{\l} \transp{\Ph} \Ph \U{\l}} \norm[2]{\resp} \label{eq:lbrhs380}\\
&\geq
\frac{1}{4} \sqrt{\frac{\log \rho_{\l}}{d_{\l}}} (1-\delta) \norm[2]{\resp}, 
\label{eq:lbrhs38}
\end{align}
where \eqref{eq:lbrhs380} and \eqref{eq:lbrhs38} hold on the events 
\[
\Ec \defeq \left\{ \max_{j \in [n_\l] \colon j \neq i} \abs{\xT{j}{\l} \Ph \U{\l} \vv}
 > \frac{1}{4} \sqrt{\frac{\log \rho_\l}{d_\l}} \sigmin{\UT{\l} \transp{\Ph} \Ph \U{\l}} \norm[2]{\vv}, \; \forall \vv \in \reals^{d_\l} \right\} 
\]
and
\[
\Eb \defeq \left\{ \min_{\l} \sigmin{\UT{\l} \transp{\Ph} \Ph \U{\l}} > 1 - \delta \right\}, \qquad \delta \in (0,1),
\]
respectively. 
Note that $\resp$ in \eqref{eq:lbrhs380} not being statistically independent of the $\x{j}{\l}$, $j \neq i$, is not an issue as we consider \eqref{eq:lbrhs380} on the event $\Ec$ and the inequality in the definition of $\Ec$ applies to \emph{all} $\vv \in \reals^{d_\l}$. 
Since $\mV^{(\l)} = \Ph \mU^{(\l)}$ has full rank on $\Eb$, reduced OMP terminates after $\pit_{\max} \wedge d_\l$ iterations. To see this, simply note that for $\mV^{(\l)}$ of full rank we need exactly $d_\l$ points from $\X_\l \! \setminus \! \{ \x{i}{\l} \}$ to represent $\x{i}{\l} = \mV^{(\l)} \va_i^{(\l)}$ (owing to the fact that the $\va_j^{(\l)}$, $j \in [n_\l]$, are i.i.d.~uniform on $\US{\d_\l}$) and thus $\rl_\pit= \mathbf 0$ after exactly $d_\l$ iterations.

We continue by upper-bounding the LHS of \eqref{eq:CorrSelCond} according to
\begin{align}
\max_{k \neq \l, j} \abs{ \innerprod{\x{j}{k}}{ \rl_\pit}} 
&= \max_{k \neq \l, j} \abs{\aT{j}{k} \UT{k} \Ph^T \Ph \U{\l} \resp} \nonumber \\
&=\max_{k \neq \l, j}  \abs{\aT{j}{k} \UT{k} \U{\l} \resp +  \aT{j}{k} \UT{k} ( \Ph^T \Ph -\mI) \U{\l} \resp} \nonumber \\
&\leq \max_{k \neq \l, j} \abs{\aT{j}{k} \UT{k} \U{\l} \resp} + \max_{k \neq \l, j} \abs{\aT{j}{k} \UT{k} (\Ph^T \Ph -\mI) \U{\l} \resp}  \nonumber \\
&< 
4(3 \log N + \log \pit_{\max}) \max_{k \neq \l} \frac{\norm[F]{\UT{k} \U{\l}}}{\sqrt{d_k}\sqrt{d_\l}} \norm[2]{\resp} \nonumber \\
& \qquad + \sqrt{\frac{6 \log N  + 2 \log \pit_{\max} }{d_{\min}}} \norm[2\to 2]{\UT{k} ( \transp{\mPh} \mPh -\mI)  \mU^{(\l)}} \norm[2]{\resp} \label{eq:AffCondOMPModLHSLBb} \\
&\leq 4(3 \log N + \log \pit_{\max}) \max_{k \neq \l} \frac{\norm[F]{\UT{k} \U{\l}}}{\sqrt{d_k}\sqrt{d_\l}} \norm[2]{\resp} \nonumber \\
&\qquad + \sqrt{\frac{6 \log N  + 2 \log \pit_{\max} }{d_{\min}}} \delta \norm[2]{\resp} \label{eq:AffCondOMPModLHSLBa} \\
&\leq \frac{1}{4} \sqrt{\frac{\log \rho_\l}{d_\l}} (1-\delta) \norm[2]{\resp} .
 \label{eq:assasnedomp} 
\end{align}
Here, \eqref{eq:AffCondOMPModLHSLBb} holds on the intersection of the events 
\begin{align*}
\Ee &\defeq  \Bigg\{ \max_{k \neq \l, j} \abs{\aT{j}{k} \UT{k} (\Ph^T \Ph -\mI) \U{\l} \resp}  \\
 & \qquad \qquad \qquad \qquad  \leq \sqrt{\frac{6 \log N  + 2 \log \pit_{\max} }{d_{\min}}} \norm[2\to 2]{ \transp{\mU^{(k)}} ( \transp{\mPh} \mPh -\mI)  \mU^{(\l)}} \norm[2]{\resp} \Bigg\}, \\
 \Ed &\defeq \left\{ \max_{k \neq \l, j} \abs{\aT{j}{k} \UT{k} \U{\l} \resp} 
<
4 (3\log N + \log \pit_{\max}) \, \max_{k \neq \l} \frac{\norm[F]{\transp{\U{k}} \U{\l}}}{\sqrt{d_k}\sqrt{d_\l}} \norm[2]{\resp} \right\}
\end{align*}
and \eqref{eq:AffCondOMPModLHSLBa} holds on the event
\begin{align*}
\Ea &\defeq \left\{  \max_{k, \l \colon k \neq \l} \norm[2\to 2]{ \transp{\mU^{(k)}} ( \transp{\mPh} \mPh -\mI)  \mU^{(\l)}}  < \delta \right\}. \\
\end{align*}
Recall that the notation $\resp$ does not reflect dependence on the index $i$. We do, however, make the dependence of $\Ee$ and $\Ed$ on $i$ explicit.

Finally, setting $\delta \defeq \sqrt{ \frac{28d_{\max} + 8 \log L + 2\tau}{3\tilde c p} }$ in \eqref{eq:AffCondOMPModLHSLBa}, \eqref{eq:assasnedomp} follows from assumption \eqref{eq:ThmAffConditionOMPSSCRP}. 
This is seen as follows: 
\begin{align}
 \max_{k \neq \l} \frac{\norm[F]{\transp{\U{k}} \U{\l}}}{\sqrt{d_k}} 
+
\frac{\delta \sqrt{ \frac{d_\l }{ d_{\min}}} }{2\sqrt{ 6\log N + 2\log \pit_{\max} }} 
&\leq 
\max_{k,\l \colon  k \neq \l} 
\aff(\cS_k,\cS_\l)  
+
\frac{\delta}{2} \sqrt{ \frac{d_{\max} }{ d_{\min}}}  \label{eq:useassomp0} \\
&\leq
\frac{3}{200} \frac{\sqrt{\log \rho_{\min} }}{\log N} 
\label{eq:useassomp1} \\
&\leq
\frac{\sqrt{\log \rho_{\min} }}{50 (\log N + (\log\pit_{\max})/3)} 
\label{eq:useassomp} \\
&\leq
 \frac{\sqrt{\log \rho_\l} }{   48 (\log N + (\log\pit_{\max})/3) } (1-\delta), \label{eq:useassomp2}
\end{align}
where \eqref{eq:useassomp1} is by \eqref{eq:ThmAffConditionOMPSSCRP} and \eqref{eq:useassomp} follows by noting that $(200/3) \log N = 50(\log N + (\log N)/3) > 50(\log N + (\log \pit_{\max})/3)$.
Furthermore, we have 
\begin{equation}
\frac{ \sqrt{ \log \rho_{\min}} }{ \log N + (\log\pit_{\max})/3} \leq 1 \label{eq:OMPCCRHSUB}
\end{equation}
 as a consequence of $\rho_{\min} = \min_\l (n_\l-1) /\d_\l < N/d_{\min}$, $N \geq 3$, and $d_{\min} \geq 1$. Next, \eqref{eq:OMPCCRHSUB} combined with $\sqrt{d_{\min}/d_{\max}} \leq 1$, $\max_{k,\l \colon  k\neq \l} \aff(\cS_k,\cS_\l) \geq 0$, and \eqref{eq:useassomp}, implies that $\delta \leq \frac{2}{50}$, which yields $\frac{1}{50} \leq \frac{1}{48}(1-\delta)$ and hence establishes \eqref{eq:useassomp2}.
Finally, \eqref{eq:assasnedomp} is obtained by rewriting the relation between the RHS of \eqref{eq:useassomp0} and \eqref{eq:useassomp2}. 

Note that the lower bound \eqref{eq:lbrhs38} on the RHS of \eqref{eq:CorrSelCond} and the upper bound \eqref{eq:assasnedomp} on the LHS of \eqref{eq:CorrSelCond} are equal; we have therefore established that, for fixed $(i,\l)$, $\rl_s$ obeys \eqref{eq:CorrSelCond} on $\Ec \cap \Eb \cap \Ee$ $\cap \, \Ed \cap \, \Ea$. It finally follows that on the event
$\Estartot \defeq \bigcap_{\l, i, \pit} \Ec \cap \Eb \cap \Ee \cap \Ed \cap \Ea$,
the graph $G$ obtained by SSC-OMP applied to the full data set $\X \! \setminus \! \{ \x{i}{\l} \}$ has no false connections. It remains to lower-bound $\PR{\Estartot}$. Specifically, we have
\begin{align}
\PR{\Estartot} &= 1 - \PR{\comp{\Estartot}} \nonumber \\
 &\geq 1 - \PR{\comp{\Eb}} - \PR{\comp{\Ea}} - \sum_{\l \in [L], i \in [n_\l]} \Bigg( \PR{\Ecc} + \sum_{\pit \in [\pit_{\max} \wedge d_\l]} \left( \PR{\Eec} + \PR{\Edc} \right) \Bigg) \nonumber \\
 &\geq  1 - 4e^{- \tau/2} - \sum_{\l \in [L]}n_\l e^{-\sqrt{\rho_\l} d_\l} - \frac{4}{N}, \label{eq:NoFalseConProbOMPend}
\end{align}
where the last inequality follows from 
\begin{align}
  \PR{\Ecc} &\leq e^{-\sqrt{\rho_\l} d_\l}
    \label{eq:Step3ProbOMPRP}  \\
    \PR{\comp{\Eb}} &\leq 2 e^{-\tau/2}
    \label{eq:ProbE2OMP} \\
  \PR{\Eec} &\leq \frac{2}{\pit_{\max} N^2}
    \label{eq:Step2ProbOMPRP} \\
  \PR{\Edc} &\leq \frac{2}{\pit_{\max} N^2}
    \label{eq:Step1ProbOMPRP} \\
\PR{\comp{\Ea}} &\leq 2 e^{-\tau/2}.
    \label{eq:ProbE1OMP}    
\end{align}
Here, \eqref{eq:ProbE1OMP} corresponds to \eqref{eq:prass2}, while the proofs of \eqref{eq:Step3ProbOMPRP}--\eqref{eq:Step1ProbOMPRP} are presented below.

\paragraph{Proof of \eqref{eq:Step3ProbOMPRP}:}

Since $\A{\l}_{-i}$ has full column rank with probability $1$, 
it follows from Lemma \ref{le:NinfLB} below that 
\begin{align}
\norm[\infty]{\AT{\l}_{-i} \UT{\l} \transp{\Ph} \Ph \U{\l} \vv} 
&\geq r(\mathcal{P}(\A{\l}_{-i})) \norm[2]{\UT{\l} \transp{\Ph} \Ph \U{\l} \vv} \nonumber \\
&\geq r(\mathcal{P}(\A{\l}_{-i})) \, \sigmin{\UT{\l} \transp{\Ph} \Ph \U{\l}} \norm[2]{\vv}, \nonumber
\end{align}
 for all $\vv \in \reals^{d_\l}$. We therefore have
\begin{align}
\PR{\Ecc}
&= \PR{\norm[\infty]{\AT{\l}_{-i} \UT{\l} \transp{\Ph} \Ph \U{\l} \vv} \leq \frac{1}{4} \sqrt{\frac{\log \rho_\l}{d_\l}} \sigmin{\UT{\l} \transp{\Ph} \Ph \U{\l}} \norm[2]{\vv}} \nonumber \\ 
&\leq \PR{r(\mathcal{P}(\A{\l}_{-i}))  \leq \frac{1}{4} \sqrt{\frac{\log \rho_\l}{d_\l}} } \nonumber \\
&\leq e^{-\sqrt{\rho_\l} d_\l}, \label{eq:useinrba}
\end{align}
where \eqref{eq:useinrba} follows from \eqref{eq:inradbound}, 
which uses the assumption $(n_\l-1)/ d_\l = \rho_\l \geq \rho_0>1$. 

\begin{lemma}
\label{le:NinfLB}
For a matrix $\mA \in \reals^{m \times n}$ of full column rank and $\vv \in \reals^m$, it holds that
 \begin{equation}
 \label{eq:NinfLB}
  \norm[\infty]{\transp{\mA} \vv} \geq r(\mathcal{P}(\mA)) \norm[2]{\vv},
 \end{equation}
 where $r(\mathcal{P}(\mA))$ is the inradius of the symmetrized convex hull $\mathcal{P}(\mA)$ of the columns of $\mA$.
\end{lemma}
\begin{proof}
The inequality \eqref{eq:NinfLB} obviously holds for $\vv = \mathbf 0$. Pick any $\vv \in \reals^m \! \setminus \! \{ \mathbf 0 \}$ and take $\epsilon \in (0,1)$. Let $\eta = \epsilon \norm[2]{\vv} r(\mathcal{P}(\mA))$ and assume that $\vv \in \eta  \mathcal{P}^\circ ( \mA ) = \, \{\vz \colon \norm[\infty]{\transp{\mA} \vz} \leq \eta\}$, i.e., $\vv$ is an element of the $\eta$-scaled version of the polar set $\mathcal{P}^\circ ( \mA )$. Note that $\eta > 0$ as $\norm[2]{\vv} > 0$, $\epsilon > 0$, and $r(\mathcal{P}(\mA)) > 0$ thanks to $\mA$ having full column rank. It follows from 
\cite[Thm.~1.2]{gritzmann_inner_1992} that
\begin{equation}
\frac{\norm[2]{\vv}}{\eta} \leq R( \mathcal{P}^\circ(\mA) ) = \frac{1}{r(\mathcal{P}(\mA))}. \label{eq:epsLB}
\end{equation}
Now, owing to $\eta = \epsilon \norm[2]{\vv} r(\mathcal{P}(\mA))$, \eqref{eq:epsLB} implies that $\epsilon \geq 1$, which contradicts $\epsilon \in (0,1)$.
It therefore follows that $\vv \in \reals^m \! \setminus \! \{\eta  \mathcal{P}^\circ ( \mA ) \}$ for all $\epsilon \in (0,1)$, which in turn implies that $\norm[\infty]{\transp{\mA} \vv} > \eta = \epsilon \norm[2]{\vv} r(\mathcal{P}(\mA))$ for all $\epsilon \in (0,1)$. In particular, letting $\epsilon \to 1$ yields $\norm[\infty]{\transp{\mA} \vv} \geq \norm[2]{\vv} r(\mathcal{P}(\mA))$ as desired.
\end{proof}

\paragraph{Proof of \eqref{eq:ProbE2OMP}:}

With $\sigmin{\mA} = \norm[2 \to 2]{\inv{\mA}}^{-1}$ \cite[Sec. 5.2.1]{vershynin_introduction_2012} for a full rank matrix $\mA \in \reals^{m \times m}$ it follows that
\begin{align*}
\PR{\comp{\Eb}} &= \PR{\min_{\l} \norm[2 \to 2]{\inv{(\UT{\l} \transp{\Ph} \Ph \U{\l} )}}^{-1} \leq 1 - \delta} \\ 
&= \PR{\max_{\l} \norm[2 \to 2]{\inv{(\UT{\l} \transp{\Ph} \Ph \U{\l})}} \geq \frac{1}{1 - \delta}} \\ 
&\leq 2 e^{-\tau/2},
\end{align*}
where $\tau >0$ is the numerical constant in Theorem \ref{th:OMPRPnoiseless} and the last inequality is thanks to \eqref{eq:prass1}. 

\paragraph{Proof of \eqref{eq:Step2ProbOMPRP}:}
By the union bound
\begin{align}
\PR{\Eec} &\leq   \nonumber \sum_{k \neq \l,j} \mathrm{P} \bigg[ \abs{\aT{j}{k} \UT{k} (\Ph^T \Ph -\mI) \U{\l} \resp } \nonumber \\
& \qquad \qquad \qquad \qquad > \sqrt{\frac{6 \log N  + 2 \log \pit_{\max} }{d_{\min}}} \norm[2 \to 2]{ \UT{k} (\Ph^T \Ph -\mI) \U{\l}} \norm[2]{\resp} \bigg] \nonumber \\
&\leq  \sum_{k \neq \l,j} \mathrm{P} \bigg[ \abs{ \aT{j}{k} \UT{k} (\Ph^T \Ph -\mI) \U{\l} \resp } \nonumber \\
& \qquad \qquad \qquad \qquad > \sqrt{\frac{6 \log N  + 2 \log \pit_{\max} }{d_k}}  \norm[2]{ \UT{k} (\Ph^T \Ph -\mI) \U{\l} \resp} \bigg] \nonumber \\
&\leq \sum_{k \neq \l, j} \frac{2}{\pit_{\max} N^3} \leq \frac{2}{\pit_{\max} N^2},   \label{eq:P2byadf}
\end{align}
where \eqref{eq:P2byadf} follows from Proposition \ref{thm:hoeffsphere} with $\va = \a{j}{k}$, $\vb = \UT{k} (\Ph^T \Ph -\mI) \U{\l} \resp$, and $\beta = \sqrt{6 \log N + 2 \log \pit_{\max}}$.

\paragraph{Proof of \eqref{eq:Step1ProbOMPRP}:}
We first show that $\resp / \norm[2]{\smash{\resp}}$ is distributed uniformly at random on $\US{d_\l}$; \eqref{eq:Step1ProbOMPRP} then follows 
by application of Lemma~\ref{lem:maconclem}. 

Recall that we consider reduced OMP, which computes a sparse representation of $\x{i}{\l} = \Ph \U{\l} \a{i}{\l}$ using the columns of $\mX^{(\l)}_{-i} = \Ph \U{\l} \A{\l}_{-i}$ as dictionary elements, i.e., $\Lambda_\pit$ and $\resp$ depend only on the random quantities $\Ph \mU^{(\l)}$, $\a{i}{\l}$, and $\A{\l}_{-i}$. In order to reflect 
these restricted dependencies, we write
$\resp = \resp(\Ph \mU^{(\l)},\a{i}{\l},\A{\l}_{-i})$ 
and $\Lambda_\pit = \Lambda_\pit(\Ph \mU^{(\l)},\a{i}{\l},\A{\l}_{-i}, )$.
Here, the first argument specifies the basis matrix of the data points, the second argument corresponds to the coefficient vector of the data point (in the basis specified by the first argument) a sparse representation is to be computed for, and the third argument designates the coefficient matrix of the dictionary elements (again in the basis specified by the first argument).

We start by showing that the distribution of $\resp$ is rotationally invariant. For a deterministic unitary matrix $\mW \in \reals^{d_\l \times d_\l}$, we have
\[
  \Lambda_\pit( \Ph \U{\l} \transp{\mW}, \mW \a{i}{\l}, \mW \A{\l}_{-i} ) = \Lambda_\pit(\Ph \U{\l}, \a{i}{\l}, \A{\l}_{-i}) 
\]
as the $\x{j}{\l}$ can be written as $\x{j}{\l} = \Ph \U{\l} \a{j}{\l} = \Ph \U{\l} \transp{\mW} \mW \a{j}{\l}$.

Using the shorthand notation $\Lambda'_\pit$ for $\Lambda_\pit( \Ph \U{\l} \transp{\mW}, \mW \a{i}{\l}, \mW \A{\l}_{-i} )$ and recalling that $\resp = (\mI - \A{\l}_{\Lambda_\pit} \pinv{ (\Ph \U{\l} \A{\l}_{\Lambda_\pit})} \Ph \U{\l}) \a{i}{\l}$, it follows that 
\begin{align*}
\resp(\Ph \U{\l} \transp{\mW}, \mW \a{i}{\l}, \mW \A{\l}_{-i}) &= \Big(\mbf{I} - \mW \A{\l}_{\Lambda'_\pit} \Big(\Ph \U{\l} \transp{\mW} \mW \A{\l}_{\Lambda'_\pit}\Big)^\dagger \Ph \U{\l} \transp{\mW}\Big)
\mW \a{i}{\l} \nonumber \\
&= \Big(\mbf{I} - \mW \A{\l}_{\Lambda_\pit} \Big(\Ph \U{\l} \transp{\mW} \mW \A{\l}_{\Lambda_\pit} \Big)^\dagger \Ph \U{\l} \transp{\mW} \Big) \mW \a{i}{\l} \nonumber \\
&= \mW \Big( \mI - \A{\l}_{\Lambda_\pit} \Big( \Ph \U{\l} \A{\l}_{\Lambda_\pit} \Big)^\dagger \Ph \U{\l} \Big) \a{i}{\l} \nonumber \\
&= \mW \resp( \Ph \U{\l}, \a{i}{\l}, \A{\l}_{-i}).
\end{align*}
By rotational invariance of the distributions of $\a{i}{\l}, \A{\l}_{-i}$, and $\Ph$ (by assumption in Theorem \ref{th:OMPRPnoiseless}), we have $\mW \a{i}{\l} \sim \a{i}{\l}, \mW \A{\l}_{-i} \sim \A{\l}_{-i}$, and $\Ph \U{\l} \transp{\mW} \sim \Ph \U{\l}$ (because $\text{span}(\U{\l} \transp{\mW}) = \text{span}(\U{\l})$ and the columns of $\U{\l} \transp{\mW}$ are orthonormal). We therefore get
\begin{align}
\resp(\Ph \U{\l}, \a{i}{\l}, \A{\l}_{-i}) &\sim \resp( \Ph \U{\l} \transp{\mW}, \mW \a{i}{\l}, \mW \A{\l}_{-i}) \nonumber \\
& = \mW \resp(\Ph \U{\l}, \a{i}{\l}, \A{\l}_{-i}). \label{eq:RotInvRes}
\end{align}
Since \eqref{eq:RotInvRes} holds for all unitary matrices $\mW$, the distribution of $\resp$ is rotationally invariant and $\resp / \norm[2]{\smash{\resp}}$ is, indeed, distributed uniformly on $\US{d_\l}$. 
We finally exploit this property of $\resp$ to upper-bound $\PR{\Edc}$ as follows. A union bound over all $k$, $k \neq \l$, yields
\begin{align}
\PR{\Edc} 
&\leq \sum_{k \neq \l} \PR{ \norm[\infty]{\AT{k} \UT{k} \U{\l} \resp} \geq 4(3\log N + \log \pit_{\max}) \frac{\norm[F]{\UT{k} \U{\l}}}{\sqrt{d_k}\sqrt{d_\l}} \norm[2]{\resp}} \nonumber \\
&\leq \sum_{k \neq \l} \frac{n_k + 1}{\pit_{\max} N^3} = \frac{N - n_\l + L - 1}{\pit_{\max} N^3} < \frac{N+L}{\pit_{\max} N^3} \leq \frac{2}{\pit_{\max} N^2}, \label{eq:fiubP1}
\end{align}
where \eqref{eq:fiubP1} follows by application of  
Lemma \ref{lem:maconclem} with $\mL = \A{k}$, $\va = \resp / \norm[2]{\smash{\resp}}$, $\mB =\UT{k} \U{\l}$, and $c=4(3\log N + \log \pit_{\max})$.



\end{document}